\documentclass[review,11pt]{ReportTemplate}
\usepackage{bm}
\usepackage{graphicx}

\newtheorem{theorem}{Theorem}
\newtheorem{corollary}{Corollary}
\newtheorem{definition}{Definition}
\newtheorem{lemma}{Lemma}
\newenvironment{proof}{\textbf{Proof:}\ }{\hspace{\stretch{1}}$\square$\\}

\begin{document}
\begin{frontmatter}
\title{Multi-View Active Learning in the Non-Realizable Case}

\author{Wei Wang}
\author{Zhi-Hua Zhou\corref{cor1}}
\address{National Key Laboratory for Novel Software Technology\\
Nanjing University, Nanjing 210093, China} \cortext[cor1]{\small Corresponding author.
Email: zhouzh@nju.edu.cn}

\begin{abstract}
The sample complexity of active learning under the realizability assumption has been
well-studied. The realizability assumption, however, rarely holds in practice. In this
paper, we theoretically characterize the sample complexity of active learning in the
non-realizable case under multi-view setting. We prove that, with unbounded Tsybakov
noise, the sample complexity of multi-view active learning can be $\widetilde{O}(\log
\frac{1}{\epsilon})$, contrasting to single-view setting where the polynomial
improvement is the best possible achievement. We also prove that in general multi-view
setting the sample complexity of active learning with unbounded Tsybakov noise is
$\widetilde{O}(\frac{1}{\epsilon})$, where the order of $1/\epsilon$ is independent of
the parameter in Tsybakov noise, contrasting to previous polynomial bounds where the
order of $1/\epsilon$ is related to the parameter in Tsybakov noise.
\end{abstract}


\end{frontmatter}

\section{Introduction}
In active learning \cite{David94,DasguptaNIPS07,Freund97}, the learner draws unlabeled
data from the unknown distribution defined on the learning task and actively queries
some labels from an \textit{oracle}. In this way, the active learner can achieve good
performance with much fewer labels than \textit{passive learning}. The number of these
queried labels, which is necessary and sufficient for obtaining a good leaner, is
well-known as the \textit{sample complexity} of active learning.

Many theoretical bounds on the sample complexity of active learning have been derived
based on the \textit{realizability} assumption (i.e., there exists a hypothesis
perfectly separating the data in the hypothesis class)
\cite{Balcan2007,BalcanCOLT2008,Dasgupta05,DasguptaNIPS05,Dasgupta2005,Freund97}. The
realizability assumption, however, rarely holds in practice. Recently, the sample
complexity of active learning in the \textit{non-realizable} case (i.e., the data
cannot be perfectly separated by any hypothesis in the hypothesis class because of the
noise) has been studied \cite{BalcanBL06,DasguptaNIPS07,Hanneke07}. It is worth noting
that these bounds obtained in the non-realizable case match the lower bound
$\Omega(\frac{\eta^{2}}{\epsilon^{2}})$ \cite{Kaariainen06}, in the same order as the
upper bound $O(\frac{1}{\epsilon^{2}})$ of passive learning ($\eta$ denotes the
generalization error rate of the optimal classifier in the hypothesis class and
$\epsilon$ bounds how close to the optimal classifier in the hypothesis class the
active learner has to get). This suggests that perhaps active learning in the
non-realizable case is not as efficient as that in the realizable case. To improve the
sample complexity of active learning in the non-realizable case remarkably, the model
of the noise or some assumptions on the hypothesis class and the data distribution must
be considered. Tsybakov noise model \cite{Tsybakov04} is more and more popular in
theoretical analysis on the sample complexity of active learning. However, existing
result \cite{CastroN08} shows that obtaining \textit{exponential} improvement in the
sample complexity of active learning with unbounded Tsybakov noise is hard.

Inspired by \cite{WangZ08} which proved that \textit{multi-view} setting
\cite{Blum:Mitchell1998} can help improve the sample complexity of active learning in
the realizable case remarkably, we have an insight that multi-view setting will also
help active learning in the non-realizable case. In this paper, we present the first
analysis on the sample complexity of active learning in the non-realizable case under
multi-view setting, where the non-realizability is caused by Tsybakov noise.
Specifically:

\hspace{+1em}-We define \textit{$\alpha$-expansion}, which extends the definition in
\cite{Balcan:Blum:Yang2005} and \cite{WangZ08} to the non-realizable case, and
$\beta$-condition for multi-view setting.

\hspace{+1em}-We prove that the sample complexity of active learning with Tsybakov
noise under multi-view setting can be improved to $\widetilde{O}(\log
\frac{1}{\epsilon})$ when the learner satisfies non-degradation condition.\footnote{The
$\widetilde{O}$ notation is used to hide the factor $\log\log(\frac{1}{\epsilon})$.}
This \textit{exponential} improvement holds no matter whether Tsybakov noise is bounded
or not, contrasting to single-view setting where the \textit{polynomial} improvement is
the best possible achievement for active learning with unbounded Tsybakov noise.

\hspace{+1em}-We also prove that, when non-degradation condition does not hold, the
sample complexity of active learning with unbounded Tsybakov noise under multi-view
setting is $\widetilde{O}(\frac{1}{\epsilon})$, where the order of $1/\epsilon$ is
independent of the parameter in Tsybakov noise, i.e., the sample complexity is always
$\widetilde{O}(\frac{1}{\epsilon})$ no matter how large the unbounded Tsybakov noise
is. While in previous \textit{polynomial} bounds, the order of $1/\epsilon$ is related
to the parameter in Tsybakov noise and is larger than 1 when unbounded Tsybakov noise
is larger than some degree (see Section 2). This discloses that, when non-degradation
condition does not hold, multi-view setting is still able to lead to a faster
convergence rate and our \textit{polynomial} improvement in the sample complexity is
better than previous \textit{polynomial} bounds when unbounded Tsybakov noise is large.

The rest of this paper is organized as follows. After introducing related work in
Section 2 and preliminaries in Section 3, we define $\alpha$-expansion in the
non-realizable case in Section 4. Then we analyze the sample complexity of active
learning with Tsybakov noise under multi-view setting with and without the
non-degradation condition in Section 5 and Section 6, respectively, and verify the
improvement in the sample complexity empirically in Section 7. Finally we conclude the
paper in Section 8.

\section{Related Work}
Generally, the non-realizability of learning task is caused by the presence of noise.
For learning the task with arbitrary forms of noise, Balcan et al. \cite{BalcanBL06}
proposed the agnostic active learning algorithm $A^{2}$ and proved that its sample
complexity is $\widehat{O}(\frac{\eta^{2}}{\epsilon^{2}})$.\footnote{The $\widehat{O}$
notation is used to hide the factor $polylog(\frac{1}{\epsilon})$.} Hoping to get
tighter bound on the sample complexity of the algorithm $A^{2}$, Hanneke
\cite{Hanneke07} defined the \textit{disagreement coefficient} $\theta$, which depends
on the hypothesis class and the data distribution, and proved that the sample
complexity of the algorithm $A^{2}$ is
$\widehat{O}(\theta^{2}\frac{\eta^{2}}{\epsilon^{2}})$. Later, Dasgupta et al.
\cite{DasguptaNIPS07} developed a general agnostic active learning algorithm which
extends the scheme in \cite{David94} and proved that its sample complexity is
$\widehat{O}(\theta\frac{\eta^{2}}{\epsilon^{2}})$.

Recently, the popular Tsybakov noise model \cite{Tsybakov04} was considered in
theoretical analysis on active learning and there have been some bounds on the sample
complexity. For some simple cases, where Tsybakov noise is bounded, it has been proved
that the \textit{exponential} improvement in the sample complexity is possible
\cite{Balcan2007,CastroAllerton06,Hanneke2009}. As for the situation where Tsybakov
noise is unbounded, only \textit{polynomial} improvement in the sample complexity has
been obtained. Balcan et al. \cite{Balcan2007} assumed that the samples are drawn
uniformly from the the unit ball in $R^{d}$ and proved that the sample complexity of
active learning with unbounded Tsybakov noise is
$O\big(\epsilon^{-\frac{2}{1+\lambda}}\big)$ ($\lambda>0$ depends on Tsybakov noise).
This uniform distribution assumption, however, rarely holds in practice. Castro and
Nowak \cite{CastroN08} showed that the sample complexity of active learning with
unbounded Tsybakov noise is
$\widehat{O}\big(\epsilon^{-\frac{2\mu\omega+d-2\omega-1}{\mu\omega}}\big)$ ($\mu>1$
depends on another form of Tsybakov noise, $\omega\geq 1$ depends on the H{\"o}lder
smoothness and $d$ is the dimension of the data). This result is also based on the
strong uniform distribution assumption. Cavallanti et al. \cite{CavallantiCG08} assumed
that the labels of examples are generated according to a simple linear noise model and
indicated that the sample complexity of active learning with unbounded Tsybakov noise
is $O\big(\epsilon^{-\frac{2(3+\lambda)}{(1+\lambda)(2+\lambda)}}\big)$. Hanneke
\cite{Hanneke2009} proved that the algorithms or variants thereof in \cite{BalcanBL06}
and \cite{DasguptaNIPS07} can achieve the \textit{polynomial} sample complexity
$\widehat{O}\big(\epsilon^{-\frac{2}{1+\lambda}}\big)$ for active learning with
unbounded Tsybakov noise. For active learning with unbounded Tsybakov noise, Castro and
Nowak \cite{CastroN08} also proved that at least $\Omega(\epsilon^{-\rho})$ labels are
requested to learn an $\epsilon$-approximation of the optimal classifier
($\rho\in(0,2)$ depends on Tsybakov noise). This result shows that the
\textit{polynomial} improvement is the best possible achievement for active learning
with unbounded Tsybakov noise in single-view setting. Wang \cite{wangNIPS2009}
introduced smooth assumption to active learning with \textit{approximate} Tsybakov
noise and proved that if the classification boundary and the underlying distribution
are smooth to $\xi$-th order and $\xi>d$, the sample complexity of active learning is
$\widehat{O}\big(\epsilon^{-\frac{2d}{\xi+d}}\big)$; if the boundary and the
distribution are infinitely smooth, the sample complexity of active learning is
$O\big(polylog(\frac{1}{\epsilon})\big)$. Nevertheless, this result is for
\textit{approximate} Tsybakov noise and the assumption on large smoothness order (or
infinite smoothness order) rarely holds for data with high dimension $d$ in practice.

\section{Preliminaries}

In multi-view setting, the instances are described with several different disjoint sets
of features. For the sake of simplicity, we only consider \textit{two-view} setting in
this paper. Suppose that $X=X_{1}\times X_{2}$ is the instance space, $X_{1}$ and
$X_{2}$ are the two views, $Y=\{0, 1\}$ is the label space and $\mathcal {D}$ is the
distribution over $X\times Y$. Suppose that $c=(c_{1}, c_{2})$ is the optimal Bayes
classifier, where $c_{1}$ and $c_{2}$ are the optimal Bayes classifiers in the two
views, respectively. Let $\mathcal {H}_{1}$ and $\mathcal {H}_{2}$ be the hypothesis
class in each view and suppose that $c_{1}\in \mathcal {H}_{1}$ and $c_{2}\in \mathcal
{H}_{2}$. For any instance $x=(x_{1},x_{2})$, the hypothesis $h_{v}\in \mathcal
{H}_{v}$ $(v=1,2)$ makes that $h_{v}(x_{v})=1$ if $x_{v}\in S_{v}$ and $h_{v}(x_{v})=0$
otherwise, where $S_{v}$ is a subset of $X_{v}$. In this way, any hypothesis $h_{v} \in
\mathcal {H}_{v}$ corresponds to a subset $S_{v}$ of $X_{v}$ (as for how to combine the
hypotheses in the two views, see Section 5). Considering that $x_{1}$ and $x_{2}$
denote the same instance $x$ in different views, we overload $S_{v}$ to denote the
instance set $\{x=(x_{1},x_{2}):x_{v}\in S_{v}\}$ without confusion. Let $S_{v}^{*}$
correspond to the optimal Bayes classifier $c_{v}$. It is well-known \cite{DEVROYE1996}
that $S_{v}^{*}=\{x_{v}:\varphi_{v}(x_{v})\geq \frac{1}{2}\}$, where
$\varphi_{v}(x_{v})=P(y=1|x_{v})$. Here, we also overload $S_{v}^{*}$ to denote the
instances set $\{x=(x_{1},x_{2}):x_{v}\in S_{v}^{*}\}$. The error rate of a hypothesis
$S_{v}$ under the distribution $\mathcal {D}$ is
$R(h_{v})=R(S_{v})=Pr_{(x_{1},x_{2},y)\in \mathcal {D}}\big(y\neq \textbf{I}(x_{v}\in
S_{v})\big)$. In general, $R(S_{v}^{*})\neq 0$ and the excess error of $S_{v}$ can be
denoted as follows, where $S_{v} \Delta S_{v}^{*}=(S_{v} -
S_{v}^{*})\cup(S_{v}^{*}-S_{v})$ and $d(S_{v},S_{v}^{*})$ is a pseudo-distance between
the sets $S_{v}$ and $S_{v}^{*}$.\vspace{-1mm}
\begin{eqnarray}
\label{bayes}
  R(S_{v})-R(S_{v}^{*})=\int_{S_{v} \Delta S_{v}^{*}}|2\varphi_{v}(x_{v})-1|p_{x_{v}}d_{x_{v}}
  \triangleq d(S_{v},S_{v}^{*})
\end{eqnarray}
Let $\eta_{v}$ denote the error rate of the optimal Bayes classifier $c_{v}$ which is
also called as the noise rate in the non-realizable case. In general, $\eta_{v}$ is
less than $\frac{1}{2}$. In order to model the noise, we assume that the data
distribution and the Bayes decision boundary in each view satisfies the popular
Tsybakov noise condition \cite{Tsybakov04} that $Pr_{x_{v}\in
X_{v}}(|\varphi_{v}(x_{v})-1/2|\leq t)\leq C_{0}t^{\lambda}$ for some finite $C_{0} >
0$, $\lambda > 0$ and all $0 < t \leq 1/2$, where $\lambda=\infty$ corresponds to the
best learning situation and the noise is called \textit{bounded} \cite{CastroN08};
while $\lambda=0$ corresponds to the worst situation. When $\lambda < \infty$, the
noise is called \textit{unbounded} \cite{CastroN08}. According to Proposition 1 in
\cite{Tsybakov04}, it is easy to know that (\ref{Tsycontion2}) holds.\vspace{-1mm}
\begin{eqnarray}\label{Tsycontion2}
d(S_{v},S_{v}^{*})\geq C_{1}d_{\Delta}^{k}(S_{v},S_{v}^{*})
\end{eqnarray}
Here $k=\frac{1+\lambda}{\lambda}$,
$C_{1}=2C_{0}^{-1/\lambda}\lambda(\lambda+1)^{-1-1/\lambda}$,
$d_{\Delta}(S_{v},S_{v}^{*})=Pr(S_{v}-S_{v}^{*})+ Pr(S_{v}^{*}-S_{v})$ is also a
pseudo-distance between the sets $S_{v}$ and $S_{v}^{*}$, and $d(S_{v},S_{v}^{*})\leq
d_{\Delta}(S_{v},S_{v}^{*}) \leq 1$. We will use the following lamma \cite{Anthony1999}
which gives the standard sample complexity for non-realizable learning
task.\vspace{-1mm}
\begin{lemma}\label{lemma2}
Suppose that $\mathcal {H}$ is a set of functions from $X$ to $Y=\{0, 1\}$ with finite
VC-dimension $V\geq 1$ and $\mathcal {D}$ is the fixed but unknown distribution over $X
\times Y$. For any $\epsilon$, $\delta > 0$, there is a positive constant $C$, such
that if the size of sample $\{(x^{1},y^{1}),\ldots,(x^{N},y^{N})\}$ from $\mathcal {D}$
is $N(\epsilon, \delta)=\frac{C}{\epsilon^{2}}\big(V+\log(\frac{1}{\delta})\big)$, then
with probability at least $1-\delta$, for all $h \in \mathcal {H}$, the following
holds.\vspace{-2mm}
\begin{eqnarray}
\nonumber
   |\frac{1}{N}\sum\nolimits_{i=1}^{N}\textbf{I}\big(h(x^{i})\neq y^{i}\big)-
   \textbf{E}_{(x,y)\in \mathcal {D}}\textbf{I}\big(h(x)\neq y\big)|\leq\epsilon
\end{eqnarray}
\end{lemma}

\section{$\alpha$-Expansion in the Non-realizable Case}
Multi-view active learning first described in \cite{MusleaMK02} focuses on the
\textit{contention points} (i.e., unlabeled instances on which different views predict
different labels) and queries some labels of them. It is motivated by that querying the
labels of contention points may help at least one of the two views to learn the optimal
classifier. Let $S_{1} \oplus S_{2}=(S_{1}-S_{2})\cup (S_{2}-S_{1})$ denote the
contention points between $S_{1}$ and $S_{2}$, then $Pr(S_{1} \oplus S_{2})$ denotes
the probability mass on the contentions points. ``$\Delta$'' and ``$\oplus$'' mean the
same operation rule. In this paper, we use ``$\Delta$'' when referring the excess error
between $S_{v}$ and $S_{v}^{*}$ and use ``$\oplus$'' when referring the difference
between the two views $S_{1}$ and $S_{2}$. In order to study multi-view active
learning, the properties of contention points should be considered. One basic property
is that $Pr(S_{1} \oplus S_{2})$ should not be too small, otherwise the two views could
be exactly the same and two-view setting would degenerate into single-view setting.

In multi-view learning, the two views represent the same learning task and generally
are consistent with each other, i.e., for any instance $x=(x_{1}, x_{2})$ the labels of
$x$ in the two views are the same. Hence we first assume that
$S_{1}^{*}=S_{2}^{*}=S^{*}$. As for the situation where $S_{1}^{*}\neq S_{2}^{*}$, we
will discuss on it further in Section 5.2. The instances agreed by the two views can be
denoted as $(S_{1} \cap S_{2})\cup(\overline{S_{1}} \cap \overline{S_{2}})$. However,
some of these agreed instances may be predicted different label by the optimal
classifier $S^{*}$, i.e., the instances in $(S_{1}\cap
S_{2}-S^{*})\cup(\overline{S_{1}} \cap \overline{S_{2}}-\overline{S^{*}})$.
Intuitively, if the contention points can convey some information about $(S_{1}\cap
S_{2}-S^{*})\cup(\overline{S_{1}} \cap \overline{S_{2}}-\overline{S^{*}})$, then
querying the labels of contention points could help to improve $S_{1}$ and $S_{2}$.
Based on this intuition and that $Pr(S_{1} \oplus S_{2})$ should not be too small, we
give our definition on \textit{$\alpha$-expansion} in the non-realizable
case.\vspace{-1mm}
\begin{definition}\label{definition1}
$\mathcal {D}$ is $\alpha$-expanding if for some $\alpha>0$ and any $S_{1}\subseteq
X_{1}$, $S_{2}\subseteq X_{2}$, (\ref{condition1}) holds.\vspace{-1mm}
\begin{eqnarray}
\label{condition1}
  Pr\big(S_{1} \oplus S_{2}\big) \geq \alpha \Big(Pr\big(S_{1} \cap S_{2}- S^{*}\big) +
  Pr\big(\overline{S_{1}} \cap \overline{S_{2}} - \overline{S^{*}}\big)\Big)
\end{eqnarray}
We say that $\mathcal {D}$ is $\alpha$-expanding with respect to hypothesis class
$\mathcal {H}_{1} \times \mathcal {H}_{2}$ if the above holds for all $S_{1}\in
\mathcal {H}_{1} \cap X_{1}$, $S_{2}\in \mathcal {H}_{2} \cap X_{2}$ (here we denote by
$\mathcal {H}_{v} \cap X_{v}$ the set \{$h \cap X_{v}$ : $h\in \mathcal {H}_{v}$\} for
$v=1,2$).
\end{definition}
Balcan et al. \cite{Balcan:Blum:Yang2005} also gave a definition of expansion,
$Pr(T_{1} \oplus T_{2}) \geq \alpha \min \big[Pr(T_{1} \cap T_{2}),Pr(\overline{T_{1}}
\cap \overline{T_{2}})\big]$, for realizable learning task under the assumptions that
the learner in each view is never ``confident but wrong'' and the learning algorithm is
able to learn from positive data only. Here $T_{v}$ denotes the instances which are
classified as positive confidently in each view. Generally, in realizable learning
tasks, we aim at studying the asymptotic performance and assume that the performance of
initial classifier is better than guessing randomly, i.e., $Pr(T_{v})> 1/2$. This
ensures that $Pr(T_{1} \cap T_{2})$ is larger than $Pr(\overline{T_{1}} \cap
\overline{T_{2}})$. In addition, in \cite{Balcan:Blum:Yang2005} the instances which are
agreed by the two views but are predicted different label by the optimal classifier can
be denoted as $\overline{T_{1}} \cap \overline{T_{2}}$. So, it can be found that
Definition~\ref{definition1} and the definition of expansion in
\cite{Balcan:Blum:Yang2005} are based on the same intuition that the amount of
contention points is no less than a fraction of the amount of instances which are
agreed by the two views but are predicted different label by the optimal classifiers.

\section{Multi-view Active Learning with Non-degradation Condition}
\begin{table*}[t]
\tiny{ \caption{Multi-view active learning with the non-degradation condition}
\centering
\begin{tabular}{p{15.7cm}}
\hline\noalign{\smallskip}
\small{\textbf{Input:} Unlabeled data set $\mathcal {U} = \{x^{1},x^{2},\cdots,\}$
where each example $x^{j}$ is given as a pair $(x_{1}^{j},x_{2}^{j})$}\\
\small{\textbf{Process:}}\\
\ \ \ \ \small{Query the labels of $m_{0}$ instances drawn randomly from $\mathcal {U}$
to compose the labeled data set $\mathcal {L}$}\\
\ \ \ \ \small{\textbf{iterate:} $i=0,1,\cdots,s$}\\
\ \ \ \ \ \ \ \ \small{Train the classifier $h_{v}^{i}$ ($v=1,2$) by minimizing the
empirical risk with $\mathcal {L}$ in each view:}\\
\ \ \ \ \ \ \ \ \ \ \ \ \ \ \small{$h_{v}^{i}=\arg\min_{h\in \mathcal
{H}_{v}}\sum_{(x_{1},x_{2},y)\in
\mathcal {L}}\textbf{I}(h(x_{v})\neq y)$;}\\
\ \ \ \ \ \ \ \ \small{Apply $h_{1}^{i}$ and $h_{2}^{i}$ to the unlabeled data set
$\mathcal {U}$ and find out the contention point set $\mathcal {Q}_{i}$;}\\
\ \ \ \ \ \ \ \ \small{Query the labels of $m_{i+1}$ instances drawn randomly from
$\mathcal{Q}_{i}$, then add them into $\mathcal {L}$ and delete}\\
\ \ \ \ \ \ \ \ \small{them from $\mathcal {U}$.}\\
\ \ \ \ \small{\textbf{end iterate}}\\
\small{\textbf{Output:} $h_{+}^{s}$ and $h_{-}^{s}$}\\
\noalign{\smallskip}\hline
\end{tabular}
}
\end{table*}
In this section, we first consider the multi-view learning in Table 1 and analyze
whether multi-view setting can help improve the sample complexity of active learning in
the non-realizable case remarkably. In multi-view setting, the classifiers are often
combined to make predictions and many strategies can be used to combine them. In this
paper, we consider the following two combination schemes, $h_{+}$ and $h_{-}$, for
binary classification:\vspace{-1mm}
\begin{eqnarray}
\label{eq:1}
   h_{+}^{i}(x) = \left\{ \begin{array}{ll}
1 & \textrm{if $h_{1}^{i}(x_{1}) = h_{2}^{i}(x_{2})=1$}\\
0 & \textrm{otherwise}
\end{array} \right.
~~~~h_{-}^{i}(x) = \left\{ \begin{array}{ll}
0 & \textrm{if $h_{1}^{i}(x_{1}) = h_{2}^{i}(x_{2})=0$}\\
1 & \textrm{otherwise}
\end{array} \right.
\end{eqnarray}
\subsection{The Situation Where $S_{1}^{*}=S_{2}^{*}$}
With (\ref{eq:1}), the error rate of the combined classifiers $h_{+}^{i}$ and
$h_{-}^{i}$ satisfy (\ref{combonation}) and (\ref{combonation2}),
respectively.\vspace{-1mm}
\begin{eqnarray}
\label{combonation}
  R(h_{+}^{i})-R(S^{*})=R(S_{1}^{i} \cap S_{2}^{i})-R(S^{*})
  \leq d_{\Delta}(S_{1}^{i} \cap S_{2}^{i},S^{*})\\
\label{combonation2}
  R(h_{-}^{i})-R(S^{*})= R(S_{1}^{i} \cup S_{2}^{i})-
  R(S^{*})
  \leq d_{\Delta}(S_{1}^{i} \cup S_{2}^{i},S^{*})
\end{eqnarray}
Here $S_{v}^{i}\subset X_{v}$ ($v=1,2$) corresponds to the classifier $h_{v}^{i} \in
\mathcal {H}_{v}$ in the $i$-th round. In each round of multi-view active learning,
labels of some contention points are queried to augment the training data set $\mathcal
{L}$ and the classifier in each view is then refined. As discussed in \cite{WangZ08},
we also assume that the learner in Table 1 satisfies the \textit{non-degradation}
condition as the amount of labeled training examples increases, i.e., (\ref{eq:2a})
holds, which implies that the excess error of $S_{v}^{i+1}$ is no larger than that of
$S_{v}^{i}$ in the region of $\overline{S_{1}^{i} \oplus S_{2}^{i}}$.\vspace{-1mm}
\begin{eqnarray}
\label{eq:2a}
   Pr\big(S_{v}^{i+1}\Delta S^{*}\big|\overline{S_{1}^{i} \oplus
   S_{2}^{i}}\big)\leq Pr(S_{v}^{i}\Delta S^{*}\big|\overline{S_{1}^{i} \oplus S_{2}^{i}})
\end{eqnarray}

To illustrate the non-degradation condition, we give the following example: Suppose the
data in $X_{v}$ ($v=1,2$) fall into $n$ different clusters, denoted by $\pi_{1}^{v},
\ldots, \pi_{n}^{v}$, and every cluster has the same probability mass for simplicity.
The positive class is the union of some clusters while the negative class is the union
of the others. Each positive (negative) cluster $\pi_{\xi}^{v}$ in $X_{v}$ is
associated with only $3$ positive (negative) clusters $\pi_{\varsigma}^{3-v}$ $(\xi,
\varsigma\in \{1,\ldots, n\})$ in $X_{3-v}$ (i.e., given an instance $x_{v}$ in
$\pi_{\xi}^{v}$, $x_{3-v}$ will only be in one of these $\pi_{\varsigma}^{3-v}$).
Suppose the learning algorithm will predict all instances in each cluster with the same
label, i.e., the hypothesis class $\mathcal {H}_{v}$ consists of the hypotheses which
do not split any cluster. Thus, the cluster $\pi_{\xi}^{v}$ can be classified according
to the posterior probability $P(y=1|\pi_{\xi}^{v})$ and querying the labels of
instances in cluster $\pi_{\xi}^{v}$ will not influence the estimation of the posterior
probability for cluster $\pi_{\varsigma}^{v}$ ($\varsigma\neq \xi$). It is evident that
the non-degradation condition holds in this task. Note that the non-degradation
assumption may not always hold, and we will discuss on this in Section 6. Now we give
Theorem \ref{theorem1}.

\begin{theorem}\label{theorem1}
For data distribution $\mathcal {D}$ $\alpha$-expanding with respect to hypothesis
class $\mathcal {H}_{1} \times \mathcal {H}_{2}$ according to Definition
\ref{definition1}, when the non-degradation condition holds, if
$s=\lceil\frac{2\log\frac{1}{8\epsilon}}{\log\frac{1}{C_{2}}}\rceil$ and $m_{i}
=\frac{256^{k}C}{C_{1}^{2}}\big(V+\log(\frac{16(s+1)}{\delta})\big)$, the multi-view
active learning in Table 1 will generate two classifiers $h_{+}^{s}$ and $h_{-}^{s}$,
at least one of which is with error rate no larger than $R(S^{*})+\epsilon$ with
probability at least $1-\delta$.
\\
Here, $V = \max[VC(\mathcal {H}_{1}), VC(\mathcal {H}_{2})]$ where $VC(\mathcal {H})$
denotes the VC-dimension of the hypothesis class $\mathcal {H}$,
$k=\frac{1+\lambda}{\lambda}$,
$C_{1}=2C_{0}^{-1/\lambda}\lambda(\lambda+1)^{-1-1/\lambda}$ and
$C_{2}=\frac{5\alpha+8}{6\alpha+8}$.
\end{theorem}
\vspace{-1mm}
\begin{proof}
Let $Q_{i}=S_{1}^{i} \oplus S_{2}^{i}$. First we prove that if each view $X_{v}$
($v=1,2$) satisfies Tsybakov noise condition, i.e., $Pr_{x_{v} \in
X_{v}}(|\varphi_{v}(x_{v})-1/2|\leq t)\leq C_{3}t^{\lambda_{3}}$ for some finite $C_{3}
> 0$, $\lambda_{3} > 0$ and all $0 < t \leq 1/2$, Tsybakov noise condition can also be
met in $Q_{i}$, i.e., $\frac{Pr_{x_{v} \in Q_{i}}(|\varphi_{v}(x_{v})-1/2|\leq
t)}{Pr(Q_{i})}\leq C_{4}t^{\lambda_{4}}$ for some finite $C_{4} > 0$, $\lambda_{4} > 0$
and all $0 < t \leq 1/2$. Suppose Tsybakov noise condition cannot be met in $Q_{i}$,
then for $C_{*}=\frac{C_{3}}{Pr(Q_{i})}$ and $\lambda_{*}=\lambda_{3}$, there exists
some $0< t_{*} \leq 1/2$ to satisfy that $\frac{Pr_{x_{v} \in
Q_{i}}(|\varphi_{v}(x_{v})-1/2|\leq t)}{Pr(Q_{i})}> C_{*}t_{*}^{\lambda_{*}}$. So we
get
\begin{eqnarray}
\nonumber
  Pr_{x_{v} \in X_{v}}(|\varphi_{v}(x_{v})-1/2|\leq t)\geq
  Pr_{x_{v} \in Q_{i}}(|\varphi_{v}(x_{v})-1/2|\leq t)> C_{3}t_{*}^{\lambda_{3}}.
\end{eqnarray}
It is in contradiction with that $X_{v}$ satisfies Tsybakov noise condition. Thus, we
get that Tsybakov noise condition can also be met in $Q_{i}$. Without loss of
generality, suppose that Tsybakov noise condition in all $Q_{i}$ and $X_{v}$ can be met
for the same finite $C_{0}$ and $\lambda$.

Since $m_{0} =\frac{256^{k}C}{C_{1}^{2}}\big(V+\log(\frac{16(s+1)}{\delta})\big)$,
according to Lemma~\ref{lemma2} we know that $d(S_{v}^{0},S^{*}) \leq
  \frac{C_{1}}{16^{k}}$ with probability at least $1-\frac{\delta}{16(s+1)}$.
With $d(S_{v},S_{v}^{*})\geq C_{1}d_{\Delta}^{k}(S_{v},S_{v}^{*})$, we get
$d_{\Delta}(S_{v}^{0},S^{*})\leq \frac{1}{16}$. It is easy to find that
$d_{\Delta}(S_{1}^{0}\cap S_{2}^{0},S^{*}) \leq d_{\Delta}(S_{1}^{0},S^{*})
  + d_{\Delta}(S_{2}^{0},S^{*}) \leq 1/8$ holds with probability at least $1-\frac{\delta}{8(s+1)}$.

For $i\geq 0$, $m_{i+1}$ number of labels are queried randomly from $Q_{i}$. Thus,
similarly according to Lemma~\ref{lemma2} we have $d_{\Delta}(S_{1}^{i+1}\cap
S_{2}^{i+1} \mid Q_{i},
  S^{*}\mid Q_{i}) \leq 1/8$ with probability at least $1-\frac{\delta}{8(s+1)}$. Let
$T_{v}^{i+1}=S_{v}^{i+1}\cap \overline{Q_{i}}$ and $\tau_{i+1}=\frac{Pr(T_{1}^{i+1}
\oplus T_{2}^{i+1} - S^{*})}{Pr(T_{1}^{i+1} \oplus T_{2}^{i+1})}-\frac{1}{2}$, it is
easy to get
\begin{eqnarray}
\nonumber
   Pr\big(S^{*}\cap (S_{1}^{i+1}\oplus S_{2}^{i+1})|\overline{Q_{i}}\big)-
   Pr\big(\overline{S^{*}}\cap (S_{1}^{i+1}\oplus S_{2}^{i+1})|\overline{Q_{i}}\big)
   = -2\tau_{i+1} Pr(S_{1}^{i+1}\oplus S_{2}^{i+1}|\overline{Q_{i}}).
\end{eqnarray}
Considering the non-degradation condition and $d_{\Delta}(S_{1}^{i}\cap
S_{2}^{i}|\overline{Q_{i}},S^{*}|\overline{Q_{i}})=d_{\Delta}(S_{v}^{i}|\overline{Q_{i}},S^{*}|\overline{Q_{i}})$,
we calculate that
\begin{eqnarray}
\nonumber
   && d_{\Delta}(S_{1}^{i+1}\cap S_{2}^{i+1}|\overline{Q_{i}},S^{*}|\overline{Q_{i}})\\
\nonumber
   &=&\frac{1}{2}\Big(d_{\Delta}(S_{1}^{i+1}|\overline{Q_{i}},S^{*}|\overline{Q_{i}})+
   d_{\Delta}(S_{2}^{i+1}|\overline{Q_{i}},S^{*}|\overline{Q_{i}})\Big)+
   \frac{1}{2}Pr\Big(S^{*}\cap (S_{1}^{i+1}\oplus S_{2}^{i+1})|\overline{Q_{i}}\Big)\\
\nonumber
   &&-\frac{1}{2}Pr\Big(\overline{S^{*}}\cap (S_{1}^{i+1}\oplus
   S_{2}^{i+1})|\overline{Q_{i}}\Big)\\
\nonumber
   &\leq&\frac{1}{2}\Big(d_{\Delta}(S_{1}^{i}|\overline{Q_{i}},S^{*}|\overline{Q_{i}})+
   d_{\Delta}(S_{2}^{i}|\overline{Q_{i}},S^{*}|\overline{Q_{i}})\Big)-
   \tau_{i+1}Pr(S_{1}^{i+1}\oplus S_{2}^{i+1}|\overline{Q_{i}})\\
\nonumber
   &=& d_{\Delta}(S_{1}^{i}\cap S_{2}^{i}|\overline{Q_{i}},S^{*}|\overline{Q_{i}})-
   \tau_{i+1}Pr(S_{1}^{i+1}\oplus S_{2}^{i+1}|\overline{Q_{i}}).
\end{eqnarray}
So we have
\begin{eqnarray}
\nonumber
    && d_{\Delta}(S_{1}^{i+1} \cap S_{2}^{i+1},S^{*})\\
\nonumber
    &=&d_{\Delta}(S_{1}^{i+1} \cap S_{2}^{i+1}|Q_{i}, S^{*}|Q_{i})Pr(Q_{i})+
    d_{\Delta}(S_{1}^{i+1} \cap S_{2}^{i+1}|\overline{Q_{i}},
    S^{*}|\overline{Q_{i}})Pr(\overline{Q_{i}})\\
\nonumber
    &\leq& \frac{1}{8}Pr(Q_{i})+
    d_{\Delta}(S_{1}^{i}\cap S_{2}^{i}|\overline{Q_{i}},S^{*}|\overline{Q_{i}})Pr(\overline{Q_{i}})
    -\tau_{i+1}Pr\big((S_{1}^{i+1}
    \oplus S_{2}^{i+1})\cap \overline{Q_{i}}\big).
\end{eqnarray}
Considering $d_{\Delta}(S_{1}^{i}\cap
S_{2}^{i}|\overline{Q_{i}},S^{*}|\overline{Q_{i}})Pr(\overline{Q_{i}})=Pr(S_{1}^{i}\cap
S_{2}^{i}-S^{*}) + Pr(\overline{S_{1}^{i}} \cap \overline{S_{2}^{i}} -
\overline{S^{*}})$, we have
\begin{eqnarray}
\nonumber
   &&d_{\Delta}(S_{1}^{i+1} \cap S_{2}^{i+1},S^{*})\\
\nonumber
    &\leq& Pr(S_{1}^{i}\cap S_{2}^{i}-S^{*}) +
    Pr(\overline{S_{1}^{i}} \cap \overline{S_{2}^{i}} - \overline{S^{*}})
    +\frac{1}{8}Pr(S_{1}^{i} \oplus S_{2}^{i})-\tau_{i+1}Pr\big((S_{1}^{i+1}
    \oplus S_{2}^{i+1})\cap \overline{Q_{i}}\big).
\end{eqnarray}
Similarly, we get
\begin{eqnarray}
\nonumber
   &&d_{\Delta}(S_{1}^{i+1} \cup S_{2}^{i+1},S^{*})\\
\nonumber
    &\leq& Pr(S_{1}^{i}\cap S_{2}^{i}-S^{*}) +
    Pr(\overline{S_{1}^{i}} \cap \overline{S_{2}^{i}} - \overline{S^{*}})
    +\frac{1}{8}Pr(S_{1}^{i} \oplus S_{2}^{i})+\tau_{i+1}Pr\big((S_{1}^{i+1}
    \oplus S_{2}^{i+1})\cap \overline{Q_{i}}\big).
\end{eqnarray}
Let $\gamma_{i}=\frac{Pr(S_{1}^{i} \oplus S_{2}^{i} - S^{*})}{Pr(S_{1}^{i} \oplus
S_{2}^{i})}-\frac{1}{2}$, we have
\begin{eqnarray}
\nonumber
    d_{\Delta}(S_{1}^{i} \cap S_{2}^{i},S^{*})
    &=&d_{\Delta}(S_{1}^{i} \cap S_{2}^{i}|Q_{i},
  S^{*}|Q_{i})Pr(Q_{i})+d_{\Delta}(S_{1}^{i} \cap S_{2}^{i}|\overline{Q_{i}},
  S^{*}|\overline{Q_{i}})Pr(\overline{Q_{i}})\\
\nonumber
    &=& (1/2-\gamma_{i})Pr(S_{1}^{i} \oplus S_{2}^{i})+ Pr(S_{1}^{i}\cap S_{2}^{i}-S^{*}) +
    Pr(\overline{S_{1}^{i}} \cap \overline{S_{2}^{i}} - \overline{S^{*}})
\end{eqnarray}
and $d_{\Delta}(S_{1}^{i} \cup S_{2}^{i},S^{*})=(1/2+\gamma_{i})Pr(S_{1}^{i} \oplus
S_{2}^{i})+ Pr(S_{1}^{i}\cap S_{2}^{i}-S^{*}) + Pr(\overline{S_{1}^{i}} \cap
\overline{S_{2}^{i}} - \overline{S^{*}})$.

As in each round of the multi-view active learning some contention points of the two
views are queried and added into the training set, the difference between the two views
is decreasing, i.e., $Pr(S_{1}^{i+1}\oplus S_{2}^{i+1})$ is no larger than
$Pr(S_{1}^{i}\oplus S_{2}^{i})$.

\textbf{Case 1:} If $|\tau_{i+1}|\leq \gamma_{i}$, with respect to
Definition~\ref{definition1}, we have
\begin{eqnarray}
\nonumber
    \frac{d_{\Delta}(S_{1}^{i+1}\cup S_{2}^{i+1},S^{*})}
   {d_{\Delta}(S_{1}^{i} \cup S_{2}^{i},S^{*})}
    &\leq&\frac{\frac{1}{8}Pr(S_{1}^{i} \oplus S_{2}^{i})+|\tau_{i+1}|Pr(S_{1}^{i+1}
    \oplus S_{2}^{i+1})+\frac{1}{\alpha} Pr(S_{1}^{i} \oplus
    S_{2}^{i})}{(\frac{1}{2}+\gamma_{i})Pr(S_{1}^{i} \oplus
    S_{2}^{i})+\frac{1}{\alpha} Pr(S_{1}^{i} \oplus
    S_{2}^{i})}\\
\nonumber
   &\leq&\frac{(\frac{1}{8}+\gamma_{i})Pr(S_{1}^{i} \oplus
    S_{2}^{i})+\frac{1}{\alpha} Pr(S_{1}^{i} \oplus
    S_{2}^{i})}{(\frac{1}{2}+\gamma_{i})Pr(S_{1}^{i} \oplus
    S_{2}^{i})+\frac{1}{\alpha} Pr(S_{1}^{i} \oplus
    S_{2}^{i})} \leq \frac{5\alpha+8}{8\alpha+8};
\end{eqnarray}

\textbf{Case 2:} If $-|\tau_{i+1}|>\gamma_{i}$, with respect to
Definition~\ref{definition1}, we have
\begin{eqnarray}
\nonumber
   \frac{d_{\Delta}(S_{1}^{i+1} \cap S_{2}^{i+1},S^{*})}
  {d_{\Delta}(S_{1}^{i} \cap S_{2}^{i},S^{*})}
    &\leq&\frac{\frac{1}{8}Pr(S_{1}^{i} \oplus S_{2}^{i})+|\tau_{i+1}|Pr(S_{1}^{i+1}
    \oplus S_{2}^{i+1})+\frac{1}{\alpha} Pr(S_{1}^{i} \oplus
    S_{2}^{i})}{(\frac{1}{2}+|\gamma_{i}|)Pr(S_{1}^{i} \oplus
    S_{2}^{i})+\frac{1}{\alpha} Pr(S_{1}^{i} \oplus
    S_{2}^{i})}\\
\nonumber
  &\leq& \frac{5\alpha+8}{8\alpha+8};
\end{eqnarray}

\textbf{Case 3:} If $\tau_{i+1}\geq\gamma_{i}$ and $0\leq\gamma_{i}\leq\frac{1}{4}$,
with respect to Definition~\ref{definition1}, we have
\begin{eqnarray}
\nonumber
   \frac{d_{\Delta}(S_{1}^{i+1} \cap S_{2}^{i+1},S^{*})}
  {d_{\Delta}(S_{1}^{i} \cap S_{2}^{i},S^{*})}
    &\leq&\frac{\frac{1}{8}Pr(S_{1}^{i} \oplus S_{2}^{i})+\frac{1}{\alpha} Pr(S_{1}^{i} \oplus
    S_{2}^{i})}{(\frac{1}{2}-\gamma_{i})Pr(S_{1}^{i} \oplus
    S_{2}^{i})+\frac{1}{\alpha} Pr(S_{1}^{i} \oplus
    S_{2}^{i})}\\
\nonumber
  &\leq& \frac{\alpha+8}{2\alpha+8};
\end{eqnarray}

\textbf{Case 4:} If $\tau_{i+1}\geq\gamma_{i}$ and
$\frac{1}{4}<\gamma_{i}\leq\frac{1}{2}$, with respect to Definition~\ref{definition1},
we have
\begin{eqnarray}
\nonumber
   \frac{d_{\Delta}(S_{1}^{i+1} \cup S_{2}^{i+1},S^{*})}
  {d_{\Delta}(S_{1}^{i} \cup S_{2}^{i},S^{*})}\
    &\leq&\frac{\frac{1}{8}Pr(S_{1}^{i} \oplus S_{2}^{i})+\tau_{i+1}Pr(S_{1}^{i+1}
    \oplus S_{2}^{i+1})+\frac{1}{\alpha} Pr(S_{1}^{i} \oplus
    S_{2}^{i})}{(\frac{1}{2}+\gamma_{i})Pr(S_{1}^{i} \oplus
    S_{2}^{i})+\frac{1}{\alpha} Pr(S_{1}^{i} \oplus
    S_{2}^{i})}\\
\nonumber
   &\leq& \frac{5\alpha+8}{6\alpha+8};
\end{eqnarray}

\textbf{Case 5:} If $\tau_{i+1}<\gamma_{i}$ and $-\frac{1}{4}\leq\gamma_{i} \leq 0$,
with respect to Definition~\ref{definition1}, we have
\begin{eqnarray}
\nonumber
   \frac{d_{\Delta}(S_{1}^{i+1} \cup S_{2}^{i+1},S^{*})}
  {d_{\Delta}(S_{1}^{i} \cup S_{2}^{i},S^{*})}\
    &\leq&\frac{\frac{1}{8}Pr(S_{1}^{i} \oplus S_{2}^{i})+\frac{1}{\alpha} Pr(S_{1}^{i} \oplus
    S_{2}^{i})}{(\frac{1}{2}+\gamma_{i})Pr(S_{1}^{i} \oplus
    S_{2}^{i})+\frac{1}{\alpha} Pr(S_{1}^{i} \oplus
    S_{2}^{i})}\\
\nonumber
    &\leq& \frac{\alpha+8}{2\alpha+8};
\end{eqnarray}

\textbf{Case 6:} If $\tau_{i+1}<\gamma_{i}$ and
$-\frac{1}{2}\leq\gamma_{i}<-\frac{1}{4}$, with respect to
Definition~\ref{definition1}, we have
\begin{eqnarray}
\nonumber
   \frac{d_{\Delta}(S_{1}^{i+1} \cap S_{2}^{i+1},S^{*})}
  {d_{\Delta}(S_{1}^{i} \cap S_{2}^{i},S^{*})}
    &\leq&\frac{\frac{1}{8}Pr(S_{1}^{i} \oplus S_{2}^{i})+|\tau_{i+1}|Pr(S_{1}^{i+1}
    \oplus S_{2}^{i+1})+\frac{1}{\alpha} Pr(S_{1}^{i} \oplus
    S_{2}^{i})}{(\frac{1}{2}+|\gamma_{i}|)Pr(S_{1}^{i} \oplus
    S_{2}^{i})+\frac{1}{\alpha} Pr(S_{1}^{i} \oplus
    S_{2}^{i})}\\
\nonumber
  &\leq& \frac{5\alpha+8}{6\alpha+8};
\end{eqnarray}

\textbf{Case 7:} If $\tau_{i+1}\leq-\gamma_{i}$ and $0\leq\gamma_{i}\leq\frac{1}{2}$,
with respect to Definition~\ref{definition1}, we have
\begin{eqnarray}
\nonumber
   \frac{d_{\Delta}(S_{1}^{i+1} \cup S_{2}^{i+1},S^{*})}
  {d_{\Delta}(S_{1}^{i} \cup S_{2}^{i},S^{*})}
    &\leq&\frac{\frac{1}{8}Pr(S_{1}^{i} \oplus S_{2}^{i})+\frac{1}{\alpha} Pr(S_{1}^{i} \oplus
    S_{2}^{i})}{(\frac{1}{2}+\gamma_{i})Pr(S_{1}^{i} \oplus
    S_{2}^{i})+\frac{1}{\alpha} Pr(S_{1}^{i} \oplus
    S_{2}^{i})}\\
\nonumber
  &\leq& \frac{\alpha+8}{4\alpha+8};
\end{eqnarray}

\textbf{Case 8:} If $\tau_{i+1}>-\gamma_{i}$ and $-\frac{1}{2}\leq\gamma_{i}\leq0$,
with respect to Definition~\ref{definition1}, we have
\begin{eqnarray}
\nonumber
   \frac{d_{\Delta}(S_{1}^{i+1} \cap S_{2}^{i+1},S^{*})}
  {d_{\Delta}(S_{1}^{i} \cap S_{2}^{i},S^{*})}
    &\leq&\frac{\frac{1}{8}Pr(S_{1}^{i} \oplus S_{2}^{i})+\frac{1}{\alpha} Pr(S_{1}^{i} \oplus
    S_{2}^{i})}{(\frac{1}{2}+|\gamma_{i}|)Pr(S_{1}^{i} \oplus
    S_{2}^{i})+\frac{1}{\alpha} Pr(S_{1}^{i} \oplus
    S_{2}^{i})}\\
\nonumber
  &\leq& \frac{\alpha+8}{4\alpha+8}.
\end{eqnarray}
Thus, after the $(i+1)$-th round, either $\frac{d_{\Delta}(S_{1}^{i+1} \cap
S_{2}^{i+1},S^{*})} {d_{\Delta}(S_{1}^{i} \cap S_{2}^{i},S^{*})} \leq
\frac{5\alpha+8}{6\alpha+8}$ or $\frac{d_{\Delta}(S_{1}^{i+1} \cup S_{2}^{i+1},S^{*})}
{d_{\Delta}(S_{1}^{i} \cup S_{2}^{i},S^{*})} \leq \frac{5\alpha+8}{6\alpha+8}$ holds.
Hence, we have $d_{\Delta}(S_{1}^{s} \cap S_{2}^{s},S^{*}) \leq
\frac{1}{8}\Big(\frac{5\alpha+8}{6\alpha+8}\Big)^{s/2}$ or $d_{\Delta}(S_{1}^{s} \cup
S_{2}^{s},S^{*}) \leq \frac{1}{8}\Big(\frac{5\alpha+8}{6\alpha+8}\Big)^{s/2}$ with
probability at least $1-\delta$. When
$s=\lceil\frac{2\log\frac{1}{8\epsilon}}{\log\frac{1}{C_{2}}}\rceil$, where
$C_{2}=\frac{5\alpha+8}{6\alpha+8}$ is a constant less than $1$, we have either
$d_{\Delta}(S_{1}^{s} \cap S_{2}^{s},S^{*}) \leq \epsilon$ or $d_{\Delta}(S_{1}^{s}
\cup S_{2}^{s},S^{*}) \leq \epsilon$ with probability at least $1-\delta$. Thus,
considering $R(h_{+}^{i})-R(S^{*})=R(S_{1}^{i} \cap S_{2}^{i})-R(S^{*})
  \leq d_{\Delta}(S_{1}^{i} \cap S_{2}^{i},S^{*})$ and $R(h_{-}^{i})-R(S^{*})=R(S_{1}^{i} \cup S_{2}^{i})-
  R(S^{*}) \leq d_{\Delta}(S_{1}^{i} \cup S_{2}^{i},S^{*})$,
we have either $R(h_{+}^{s})\leq R(S^{*})+\epsilon$ or $R(h_{-}^{s})\leq
R(S^{*})+\epsilon$.
\end{proof}

From Theorem~\ref{theorem1} we know that we only need to request $\sum_{i=0}^{s}m_{i}
=\widetilde{O}(\log\frac{1}{\epsilon})$ labels to learn $h_{+}^{s}$ and $h_{-}^{s}$, at
least one of which is with error rate no larger than $R(S^{*})+\epsilon$ with
probability at least $1-\delta$. If we choose $h_{+}^{s}$ and it happens to satisfy
$R(h_{+}^{s})\leq R(S^{*})+\epsilon$, we can get a classifier whose error rate is no
larger than $R(S^{*})+\epsilon$. Fortunately, there are only two classifiers and the
probability of getting the right classifier is no less than $\frac{1}{2}$. To study how
to choose between $h_{+}^{s}$ and $h_{-}^{s}$, we give Definition~\ref{definition2} at
first.\vspace{-1mm}
\begin{definition}\label{definition2}
The multi-view classifiers $S_{1}$ and $S_{2}$ satisfy $\beta$-condition if
(\ref{condition2}) holds for some $\beta>0$.
\begin{eqnarray}
\label{condition2}
   \Big|\frac{Pr\big(\{x: x\in S_{1} \oplus S_{2}\wedge y(x)=1\}\big)}{Pr(S_{1} \oplus S_{2})}
   -\frac{Pr\big(\{x: x\in S_{1} \oplus S_{2}\wedge y(x)=0\}\big)}{Pr(S_{1} \oplus S_{2})}\Big|\geq \beta
\end{eqnarray}
\end{definition}
(\ref{condition2}) implies the difference between the examples belonging to positive
class and that belonging to negative class in the contention region of $S_{1} \oplus
S_{2}$. Based on Definition~\ref{definition2}, we give Lemma~\ref{theorem1+} which
provides information for deciding how to choose between $h_{+}$ and $h_{-}$. This helps
to get Theorem~\ref{theorem1++}.\vspace{-2mm}

\begin{lemma}\label{theorem1+}
If the multi-view classifiers $S_{1}^{s}$ and $S_{2}^{s}$ satisfy $\beta$-condition,
with the number of $\frac{2\log(\frac{4}{\delta})}{\beta^{2}}$ labels we can decide
correctly whether $Pr\big(\{x: x\in S_{1}^{s} \oplus S_{2}^{s}\wedge y(x)=1\}\big)$ or
$Pr\big(\{x: x\in S_{1}^{s} \oplus S_{2}^{s}\wedge y(x)=0\}\big))$ is smaller with
probability at least $1-\delta$.
\end{lemma}\vspace{-1mm}

\begin{proof}
We apply $S_{1}^{s}$ and $S_{2}^{s}$ to the unlabeled instances set and identify the
contention point set. Then we query for labels of
$\frac{2\log(\frac{4}{\delta})}{\beta^{2}}$ instances drawn randomly from the
contention points set. With these labels we estimate the empirical value
$\widehat{P}_{1}$ of $\frac{Pr(\{x: x\in S_{1}^{s} \oplus S_{2}^{s}\wedge
y(x)=1\})}{Pr(S_{1}^{s} \oplus S_{2}^{s})}$ and the empirical value $\widehat{P}_{2}$
of $\frac{Pr(\{x: x\in S_{1}^{s} \oplus S_{2}^{s}\wedge y(x)=0\})}{Pr(S_{1}^{s} \oplus
S_{2}^{s})}$. By Chernoff bound, with number of
$\frac{2\log(\frac{4}{\delta})}{\beta^{2}}$ labels we have the following two equations
with probability at least $1-\delta$.
\begin{eqnarray}
\nonumber
  \widehat{P}_{1}\in \Big[\frac{Pr\big(\{x: x\in S_{1}^{s} \oplus S_{2}^{s}\wedge y(x)=1\}\big)}{Pr(S_{1}^{s} \oplus S_{2}^{s})}
  -\frac{\beta}{2},\frac{Pr\big(\{x: x\in S_{1}^{s} \oplus S_{2}^{s}\wedge y(x)=1\}\big)}{Pr(S_{1}^{s} \oplus S_{2}^{s})}
  +\frac{\beta}{2}\Big]\\
\nonumber
  \widehat{P}_{2}\in \Big[\frac{Pr\big(\{x: x\in S_{1}^{s} \oplus S_{2}^{s}\wedge y(x)=0\}\big)}{Pr(S_{1}^{s} \oplus S_{2}^{s})}
  -\frac{\beta}{2},\frac{Pr\big(\{x: x\in S_{1}^{s} \oplus S_{2}^{s}\wedge y(x)=0\}\big)}{Pr(S_{1}^{s} \oplus S_{2}^{s})}
  +\frac{\beta}{2}\Big]
\end{eqnarray}
If $\widehat{P}_{1}\leq \widehat{P}_{2}$, we get $Pr\big(\{x: x\in S_{1}^{s} \oplus
S_{2}^{s}\wedge y(x)=1\}\big)\leq Pr\big(\{x: x\in S_{1}^{s} \oplus S_{2}^{s}\wedge
y(x)=0\}\big)$ with probability at least $1-\delta$; otherwise, we get $Pr\big(\{x:
x\in S_{1}^{s} \oplus S_{2}^{s}\wedge y(x)=1\}\big)> Pr\big(\{x: x\in S_{1}^{s} \oplus
S_{2}^{s}\wedge y(x)=0\}\big)$ with probability at least $1-\delta$.
\end{proof}

\begin{theorem}\label{theorem1++}
For data distribution $\mathcal {D}$ $\alpha$-expanding with respect to hypothesis
class $\mathcal {H}_{1} \times \mathcal {H}_{2}$ according to Definition
\ref{definition1}, when the non-degradation condition holds, if the multi-view
classifiers satisfy $\beta$-condition, by requesting
$\widetilde{O}(\log\frac{1}{\epsilon})$ labels the multi-view active learning in Table
1 will generate a classifier whose error rate is no larger than $R(S^{*})+\epsilon$
with probability at least $1-\delta$.
\end{theorem}

\begin{proof}
According to Theorem~\ref{theorem1}, by requesting
$\widetilde{O}(\log\frac{1}{\epsilon})$ labels the multi-view active learning in Table
1 can get either $R(h_{+}^{s})\leq R(S^{*})+\epsilon$ or $R(h_{-}^{s})\leq
R(S^{*})+\epsilon$ with probability at least $1-\frac{\delta}{2}$. According to
Lemma~\ref{theorem1+}, by requesting $\frac{2\log(\frac{8}{\delta})}{\beta^{2}}$ labels
we can decide correctly whether $Pr\big(\{x: x\in S_{1}^{s} \oplus S_{2}^{s}\wedge
y(x)=1\}\big)$ or $Pr\big(\{x: x\in S_{1}^{s} \oplus S_{2}^{s}\wedge y(x)=0\}\big)$ is
smaller with probability at least $1-\frac{\delta}{2}$.

\textbf{Case 1:} If $Pr\big(\{x: x\in S_{1}^{s} \oplus S_{2}^{s}\wedge
y(x)=1\}\big)\leq Pr\big(\{x: x\in S_{1}^{s} \oplus S_{2}^{s}\wedge y(x)=0\}\big)$, we
have $R(h_{-}^{s})\leq R(h_{+}^{s})$. Thus, we get $R(h_{-}^{s})\leq R(S^{*})+\epsilon$
with probability at least $1-\delta$.

\textbf{Case 2:} If $Pr\big(\{x: x\in S_{1}^{s} \oplus S_{2}^{s}\wedge y(x)=1\}\big)>
Pr\big(\{x: x\in S_{1}^{s} \oplus S_{2}^{s}\wedge y(x)=0\}\big)$, we have
$R(h_{+}^{s})<R(h_{-}^{s})$. Thus, we get $R(h_{+}^{s})\leq R(S^{*})+\epsilon$ with
probability at least $1-\delta$.

The total number of labels to be requested is
$\widetilde{O}(\log\frac{1}{\epsilon})+\frac{2\log(\frac{8}{\delta})}{\beta^{2}}=\widetilde{O}(\log\frac{1}{\epsilon})$.
\end{proof}

From Theorem~\ref{theorem1++} we know that we only need to request
$\widetilde{O}(\log\frac{1}{\epsilon})$ labels to learn a classifier with error rate no
larger than $R(S^{*})+\epsilon$ with probability at least $1-\delta$. Thus, we achieve
an \textit{exponential} improvement in sample complexity of active learning in the
non-realizable case under multi-view setting. Sometimes, the difference between the
examples belonging to positive class and that belonging to negative class in $S_{1}^{s}
\oplus S_{2}^{s}$ may be very small, i.e., (\ref{condition3}) holds.
\begin{eqnarray}\label{condition3}
   \Big|\frac{Pr\big(\{x: x\in S_{1}^{s} \oplus S_{2}^{s}\wedge y(x)=1\}\big)}{Pr(S_{1}^{s} \oplus S_{2}^{s})}
   -\frac{Pr\big(\{x: x\in S_{1}^{s} \oplus S_{2}^{s}\wedge y(x)=0\}\big)}{Pr(S_{1}^{s} \oplus S_{2}^{s})}\Big|
   =O(\epsilon)
\end{eqnarray}
If so, we need not to estimate whether $R(h_{+}^{s})$ or $R(h_{-}^{s})$ is smaller and
Theorem~\ref{theorem1+++} indicates that both $h_{+}^{s}$ and $h_{-}^{s}$ are good
approximations of the optimal classifier.

\begin{theorem}\label{theorem1+++}
For data distribution $\mathcal {D}$ $\alpha$-expanding with respect to hypothesis
class $\mathcal {H}_{1} \times \mathcal {H}_{2}$ according to Definition
\ref{definition1}, when the non-degradation condition holds, if (\ref{condition3}) is
satisfied, by requesting $\widetilde{O}(\log\frac{1}{\epsilon})$ labels the multi-view
active learning in Table 1 will generate two classifiers $h_{+}^{s}$ and $h_{-}^{s}$
which satisfy either (a) or (b) with probability at least $1-\delta$. (a)
$R(h_{+}^{s})\leq R(S^{*})+\epsilon$ and $R(h_{-}^{s})\leq R(S^{*})+O(\epsilon)$; (b)
$R(h_{+}^{s})\leq R(S^{*})+O(\epsilon)$ and $R(h_{-}^{s})\leq R(S^{*})+\epsilon$.
\end{theorem}

\begin{proof}
Since $Pr(S_{1}^{s} \oplus S_{2}^{s})\leq 1$, with the following equation
\begin{eqnarray}
\nonumber
   \Big|\frac{Pr\big(\{x: x\in S_{1}^{s} \oplus S_{2}^{s}\wedge y(x)=1\}\big)}{Pr(S_{1}^{s} \oplus S_{2}^{s})}
   -\frac{Pr\big(\{x: x\in S_{1}^{s} \oplus S_{2}^{s}\wedge y(x)=0\}\big)}{Pr(S_{1}^{s} \oplus S_{2}^{s})}\Big|
   =O(\epsilon)
\end{eqnarray}
we have $|Pr\big(\{x: x\in S_{1}^{s} \oplus S_{2}^{s}\wedge y(x)=1\}\big)-Pr\big(\{x:
x\in S_{1}^{s} \oplus S_{2}^{s}\wedge y(x)=0\}\big)|=O(\epsilon)$. So it is easy to get
$|R(h_{+}^{s})-R(h_{-}^{s})|=O(\epsilon)$. According to Theorem~\ref{theorem1}, by
requesting $\widetilde{O}(\log\frac{1}{\epsilon})$ labels we can get either
$R(h_{+}^{s})\leq R(S^{*})+\epsilon$ or $R(h_{-}^{s})\leq R(S^{*})+\epsilon$ with
probability at least $1-\delta$. Thus, we get that $h_{+}^{s}$ and $h_{-}^{s}$ satisfy
either (a) or (b) with probability at least $1-\delta$.
\end{proof}

\subsection{The Situation Where $S_{1}^{*}\neq S_{2}^{*}$}
Although the two views represent the same learning task and generally are consistent
with each other, sometimes $S_{1}^{*}$ may be not equal to $S_{2}^{*}$. Therefore, the
\textit{$\alpha$-expansion} assumption in Definition \ref{definition1} should be
adjusted to the situation where $S_{1}^{*}\neq S_{2}^{*}$. To analyze this
theoretically, we replace $S^{*}$ by $S_{1}^{*}\cap S_{2}^{*}$ in
Definition~\ref{definition1} and get (\ref{eq6}). Similarly to Theorem~\ref{theorem1},
we get Theorem~\ref{theorem2}.\vspace{-1mm}
\begin{eqnarray}
\label{eq6}
  Pr\big(S_{1} \oplus
 S_{2}\big) \geq \alpha \Big(Pr\big(S_{1} \cap S_{2}- S_{1}^{*}\cap S_{2}^{*}\big)
 + Pr\big(\overline{S_{1}}\cap \overline{S_{2}} -
 \overline{S_{1}^{*}\cap S_{2}^{*}}\big)\Big)
\end{eqnarray}

\begin{theorem}\label{theorem2}
For data distribution $\mathcal {D}$ $\alpha$-expanding with respect to hypothesis
class $\mathcal {H}_{1} \times \mathcal {H}_{2}$ according to (\ref{eq6}), when the
non-degradation condition holds, if
$s=\lceil\frac{2\log\frac{1}{8\epsilon}}{\log\frac{1}{C_{2}}}\rceil$ and $m_{i}
=\frac{256^{k}C}{C_{1}^{2}}\big(V+\log(\frac{16(s+1)}{\delta})\big)$, the multi-view
active learning in Table 1 will generate two classifiers $h_{+}^{s}$ and $h_{-}^{s}$,
at least one of which is with error rate no larger than $R(S_{1}^{*}\cap
S_{2}^{*})+\epsilon$ with probability at least $1-\delta$. ($V$, $k$, $C_{1}$ and
$C_{2}$ are given in Theorem~\ref{theorem1}.)
\end{theorem}
\begin{proof}
Since $S_{v}^{*}$ is the optimal Bayes classifier in the $v$-th view, obviously,
$R(S_{1}^{*}\cap S_{2}^{*})$ is no less than $R(S_{v}^{*})$, $(v=1,2)$. So, learning a
classifier with error rate no larger than $R(S_{1}^{*}\cap S_{2}^{*})+\epsilon$ is not
harder than learning a classifier with error rate no larger than
$R(S_{v}^{*})+\epsilon$. Now we aim at learning a classifier with error rate no larger
than $R(S_{1}^{*}\cap S_{2}^{*})+\epsilon$. Without loss of generality, we assume
$R(S_{v}^{i})>R(S_{1}^{*}\cap S_{2}^{*})$ for $i=0,1,\ldots,s$. If $R(S_{v}^{i})\leq
R(S_{1}^{*}\cap S_{2}^{*})$, we get a classifier with error rate no larger than
$R(S_{1}^{*}\cap S_{2}^{*})+\epsilon$. Thus, we can neglect the probability mass on the
hypothesis whose error rate is less than $R(S_{1}^{*}\cap S_{2}^{*})$ and regard
$S_{1}^{*}\cap S_{2}^{*}$ as the optimal. Replacing $S^{*}$ by $S_{1}^{*}\cap
S_{2}^{*}$ in the discussion of Section 5.1, with the proof of Theorem~\ref{theorem1}
we get Theorem~\ref{theorem2} proved.
\end{proof}

Theorem \ref{theorem2} shows that for the situation where $S_{1}^{*}\neq S_{2}^{*}$, by
requesting $\widetilde{O}(\log\frac{1}{\epsilon})$ labels we can learn two classifiers
$h_{+}^{s}$ and $h_{-}^{s}$, at least one of which is with error rate no larger than
$R(S_{1}^{*}\cap S_{2}^{*})+\epsilon$ with probability at least $1-\delta$. With
Lemma~\ref{theorem1+}, we get Theorem~\ref{theorem2++} from Theorem~\ref{theorem2}.

\begin{theorem}\label{theorem2++}
For data distribution $\mathcal {D}$ $\alpha$-expanding with respect to hypothesis
class $\mathcal {H}_{1} \times \mathcal {H}_{2}$ according to (\ref{eq6}), when the
non-degradation condition holds, if the multi-view classifiers satisfy
$\beta$-condition, by requesting $\widetilde{O}(\log\frac{1}{\epsilon})$ labels the
multi-view active learning in Table 1 will generate a classifier whose error rate is no
larger than $R(S_{1}^{*}\cap S_{2}^{*})+\epsilon$ with probability at least $1-\delta$.
\end{theorem}

\begin{proof}
According to Theorem~\ref{theorem2}, by requesting
$\widetilde{O}(\log\frac{1}{\epsilon})$ labels the multi-view active learning in Table
1 can get either $R(h_{+}^{s})\leq R(S_{1}^{*}\cap S_{2}^{*})+\epsilon$ or
$R(h_{-}^{s})\leq R(S_{1}^{*}\cap S_{2}^{*})+\epsilon$ with probability at least
$1-\frac{\delta}{2}$. According to Lemma~\ref{theorem1+}, by requesting
$\frac{2\log(\frac{8}{\delta})}{\beta^{2}}$ labels we can decide correctly whether
$Pr\big(\{x: x\in S_{1}^{s} \oplus S_{2}^{s}\wedge y(x)=1\}\big)$ or $Pr\big(\{x: x\in
S_{1}^{s} \oplus S_{2}^{s}\wedge y(x)=0\}\big)$ is smaller with probability at least
$1-\frac{\delta}{2}$.

\textbf{Case 1:} If $Pr\big(\{x: x\in S_{1}^{s} \oplus S_{2}^{s}\wedge
y(x)=1\}\big)\leq Pr\big(\{x: x\in S_{1}^{s} \oplus S_{2}^{s}\wedge y(x)=0\}\big)$, we
have $R(h_{-}^{s})\leq R(h_{+}^{s})$. Thus, we get $R(h_{-}^{s})\leq R(S_{1}^{*}\cap
S_{2}^{*})+\epsilon$ with probability at least $1-\delta$.

\textbf{Case 2:} If $Pr\big(\{x: x\in S_{1}^{s} \oplus S_{2}^{s}\wedge y(x)=1\}\big)>
Pr\big(\{x: x\in S_{1}^{s} \oplus S_{2}^{s}\wedge y(x)=0\}\big)$, we have
$R(h_{+}^{s})< R(h_{-}^{s})$. Thus, we get $R(h_{+}^{s})\leq R(S_{1}^{*}\cap
S_{2}^{*})+\epsilon$ with probability at least $1-\delta$.

The total number of labels to be requested is
$\widetilde{O}(\log\frac{1}{\epsilon})+\frac{2\log(\frac{8}{\delta})}{\beta^{2}}
=\widetilde{O}(\log\frac{1}{\epsilon})$.
\end{proof}

Generally, $R(S_{1}^{*}\cap S_{2}^{*})$ is larger than $R(S_{1}^{*})$ and
$R(S_{2}^{*})$. When $S_{1}^{*}$ is not too much different from $S_{2}^{*}$, i.e.,
$Pr(S_{1}^{*}\oplus S_{2}^{*})\leq\epsilon/2$, we have Corollary \ref{corollary1} which
indicates that the \textit{exponential} improvement in the sample complexity of active
learning with Tsybakov noise is still possible.

\begin{corollary}\label{corollary1}
For data distribution $\mathcal {D}$ $\alpha$-expanding with respect to hypothesis
class $\mathcal {H}_{1} \times \mathcal {H}_{2}$ according to (\ref{eq6}), when the
non-degradation condition holds, if the multi-view classifiers satisfy
$\beta$-condition and $Pr(S_{1}^{*}\oplus S_{2}^{*})\leq\epsilon/2$, by requesting
$\widetilde{O}(\log\frac{1}{\epsilon})$ labels the multi-view active learning in Table
1 will generate a classifier with error rate no larger than $R(S_{v}^{*})+\epsilon$
($v=1,2$) with probability at least $1-\delta$.
\end{corollary}

\begin{proof}
According to Theorem~\ref{theorem2++} we know that by requesting
$\widetilde{O}(\log\frac{1}{\epsilon})$ labels the multi-view active learning in Table
1 will generate a classifier whose error rate is no larger than $R(S_{1}^{*}\cap
S_{2}^{*})+\frac{\epsilon}{2}$ with probability at least $1-\delta$. Considering that
\begin{eqnarray*}
  R(S_{1}^{*}\cap S_{2}^{*})-R(S_{v}^{*})=\int_{(S_{1}^{*}\cap S_{2}^{*}) \Delta
  S_{v}^{*}}|2\varphi_{v}(x_{v})-1|p_{x_{v}}d_{x_{v}}
  \leq Pr(S_{1}^{*}\oplus S_{2}^{*}),
\end{eqnarray*}
we have $R(S_{1}^{*}\cap S_{2}^{*})\leq R(S_{v}^{*})+\frac{\epsilon}{2}$. Thus, we get
that $R(S_{1}^{*}\cap S_{2}^{*})+\frac{\epsilon}{2}$ is no larger than
$R(S_{v}^{*})+\epsilon$.
\end{proof}

\section{Multi-view Active Learning without Non-degradation Condition}
Section 5 considers situations when the non-degradation condition holds, there are
cases, however, the non-degradation condition (\ref{eq:2a}) does not hold. In this
section we focus on the multi-view active learning in Table 2 and give an analysis with
the non-degradation condition waived. Firstly, we give Theorem~\ref{theorem3} for the
sample complexity of multi-view active learning in Table 2 when
$S_{1}^{*}=S_{2}^{*}=S^{*}$.

\begin{table*}[t]
\tiny{ \caption{Multi-view active learning without the non-degradation condition}
\centering
\begin{tabular}{p{15.7cm}}
\hline\noalign{\smallskip}
\small{\textbf{Input:} Unlabeled data set $\mathcal {U} = \{x^{1},x^{2},\cdots,\}$
where each example $x^{j}$ is given as a pair $(x_{1}^{j},x_{2}^{j})$}\\
\small{\textbf{Process:}}\\
\ \ \ \ \small{Query the labels of $m_{0}$ instances drawn randomly from $\mathcal {U}$
to compose the labeled data set $\mathcal {L}$;}\\
\ \ \ \ \small{Train the classifier $h_{v}^{0}$ ($v=1,2$) by minimizing the empirical
risk with $\mathcal {L}$ in each view:}\\
\ \ \ \ \ \ \ \ \ \ \small{$h_{v}^{0}=\arg\min_{h\in \mathcal
{H}_{v}}\sum_{(x_{1},x_{2},y)
\in \mathcal {L}}\textbf{I}(h(x_{v})\neq y)$;}\\
\ \ \ \ \small{\textbf{iterate:} $i=1,\cdots,s$}\\
\ \ \ \ \ \ \ \ \small{Apply $h_{1}^{i-1}$ and $h_{2}^{i-1}$ to the unlabeled data set
$\mathcal {U}$ and find out the contention point set $\mathcal {Q}_{i}$;}\\
\ \ \ \ \ \ \ \ \small{Query the labels of $m_{i}$ instances drawn randomly from
$\mathcal {Q}_{i}$, then add them into $\mathcal {L}$ and delete them}\\
\ \ \ \ \ \ \ \ \small{from $\mathcal {U}$;}\\
\ \ \ \ \ \ \ \ \small{Query the labels of $(2^{i}-1)m_{i}$ instances drawn randomly
from $\mathcal {U}-\mathcal {Q}_{i}$, then add them into $\mathcal {L}$ and }\\
\ \ \ \ \ \ \ \ \small{delete them from $\mathcal {U}$;}\\
\ \ \ \ \ \ \ \ \small{Train the classifier $h_{v}^{i}$ by minimizing the empirical
risk with $\mathcal {L}$ in each view:}\\
\ \ \ \ \ \ \ \ \ \ \ \ \ \ \small{$h_{v}^{i}=\arg\min_{h\in \mathcal
{H}_{v}}\sum_{(x_{1},x_{2},y)
\in \mathcal {L}}\textbf{I}(h(x_{v})\neq y)$.}\\
\ \ \ \ \small{\textbf{end iterate}}\\
\small{\textbf{Output:} $h_{+}^{s}$ and $h_{-}^{s}$}\\
\noalign{\smallskip}\hline
\end{tabular}
}
\end{table*}
\begin{theorem}\label{theorem3}
For data distribution $\mathcal {D}$ $\alpha$-expanding with respect to hypothesis
class $\mathcal {H}_{1} \times \mathcal {H}_{2}$ according to Definition
\ref{definition1}, if
$s=\lceil\frac{2\log\frac{1}{8\epsilon}}{\log\frac{1}{C_{2}}}\rceil$ and $m_{i}
=\frac{256^{k}C}{C_{1}^{2}}\big(V+\log(\frac{16(s+1)}{\delta})\big)$, the multi-view
active learning in Table 2 will generate two classifiers $h_{+}^{s}$ and $h_{-}^{s}$,
at least one of which is with error rate no larger than $R(S^{*})+\epsilon$ with
probability at least $1-\delta$. ($V$, $k$, $C_{1}$ and $C_{2}$ are given in
Theorem~\ref{theorem1}.)
\end{theorem}

\begin{proof}
After the $i$-th round in Table 2, the number of training examples in $\mathcal {L}$ is
$\sum_{b=0}^{i}2^{b}m_{i}=(2^{i+1}-1)m_{i}$. While in the $(i+1)$-th round, we randomly
query $(2^{i+1}-1)m_{i}$ labels from the region of $\overline{Q_{i}}$ and add them into
$\mathcal {L}$. So in the $(i+1)$-th round, the number of training examples for
$S_{v}^{i+1}$ $(v=1,2)$ drawn randomly from region of $\overline{Q_{i}}$ is larger than
the number of whole training examples for $S_{v}^{i}$. Since the optimal Bayes
classifier $c_{v}$ belongs to $\mathcal {H}_{v}$, according to the standard PAC-model,
it is easy to know that $d(S_{v}^{i+1}|\overline{Q_{i}},S^{*}|\overline{Q_{i}})\leq
d(S_{v}^{i}|\overline{Q_{i}},S^{*}|\overline{Q_{i}})$ can be met for any $\varphi_{v}$,
where $d(S_{v}|\overline{Q_{i}},S^{*}|\overline{Q_{i}})$ is defined as
\begin{eqnarray}
\nonumber
  d(S_{v}|\overline{Q_{i}},S^{*}|\overline{Q_{i}}) \triangleq R(S_{v}|\overline{Q_{i}})-R(S^{*}|\overline{Q_{i}})
  =\int_{(S_{v}\cap\overline{Q_{i}}) \Delta
  (S^{*}\cap\overline{Q_{i}})}|2\varphi_{v}(x_{v})-1|p_{x_{v}}d_{x_{v}}\big/Pr(\overline{Q_{i}}).
\end{eqnarray}
So, by setting $\varphi_{v}\in\{0,1\}$, we get
$d_{\Delta}(S_{v}^{i+1}|\overline{Q_{i}},S^{*}|\overline{Q_{i}})\leq
d_{\Delta}(S_{v}^{i}|\overline{Q_{i}},S^{*}|\overline{Q_{i}})$, which implies the
non-degradation condition. Thus, with the proof of Theorem~\ref{theorem1}, we get
Theorem\ref{theorem3} proved.
\end{proof}

Theorem~\ref{theorem3} shows that we can request $\sum_{i=0}^{s}2^{i}m_{i}
=\widetilde{O}(\frac{1}{\epsilon})$ labels to learn two classifiers $h_{+}^{s}$ and
$h_{-}^{s}$, at least one of which is with error rate no larger than
$R(S^{*})+\epsilon$ with probability at least $1-\delta$. To guarantee the
non-degradation condition (\ref{eq:2a}), we only need to query $(2^{i}-1)m_{i}$ more
labels in the $i$-th round. With Lemma~\ref{theorem1+}, we get
Theorem~\ref{theorem3++}.\vspace{-1mm}

\begin{theorem}\label{theorem3++}
For data distribution $\mathcal {D}$ $\alpha$-expanding with respect to hypothesis
class $\mathcal {H}_{1} \times \mathcal {H}_{2}$ according to Definition
\ref{definition1}, if the multi-view classifiers satisfy $\beta$-condition, by
requesting $\widetilde{O}(\frac{1}{\epsilon})$ labels the multi-view active learning in
Table 2 will generate a classifier whose error rate is no larger than
$R(S^{*})+\epsilon$ with probability at least $1-\delta$.
\end{theorem}

\begin{proof}
According to Theorem~\ref{theorem3}, by requesting $\widetilde{O}(\frac{1}{\epsilon})$
labels the multi-view active learning in Table 2 will generate two classifiers
$h_{+}^{s}$ and $h_{-}^{s}$, at least one of which is with error rate no larger than
$R(S^{*})+\epsilon$ with probability at least $1-\delta$. Similarly to the proof of
Theorem~\ref{theorem1++}, we get Theorem~\ref{theorem3++} proved.
\end{proof}

Theorem~\ref{theorem3++} shows that, without the non-degradation condition, we need to
request $\widetilde{O}(\frac{1}{\epsilon})$ labels to learn a classifier with error
rate no larger than $R(S^{*})+\epsilon$ with probability at least $1-\delta$. The order
of $1/\epsilon$ is independent of the parameter in Tsybakov noise. Similarly to
Theorem~\ref{theorem1+++}, we get Theorem~\ref{theorem3+++} which indicates that both
$h_{+}^{s}$ and $h_{-}^{s}$ are good approximations of the optimal classifier.

\begin{theorem}\label{theorem3+++}
For data distribution $\mathcal {D}$ $\alpha$-expanding with respect to hypothesis
class $\mathcal {H}_{1} \times \mathcal {H}_{2}$ according to Definition
\ref{definition1}, if (\ref{condition3}) holds, by requesting
$\widetilde{O}(\frac{1}{\epsilon})$ labels the multi-view active learning in Table 2
will generate two classifiers $h_{+}^{s}$ and $h_{-}^{s}$ which satisfy either (a) or
(b) with probability at least $1-\delta$. (a) $R(h_{+}^{s})\leq R(S^{*})+\epsilon$ and
$R(h_{-}^{s})\leq R(S^{*})+O(\epsilon)$; (b) $R(h_{+}^{s})\leq R(S^{*})+O(\epsilon)$
and $R(h_{-}^{s})\leq R(S^{*})+\epsilon$.
\end{theorem}

\begin{proof}
According to Theorem~\ref{theorem3}, by requesting $\widetilde{O}(\frac{1}{\epsilon})$
labels the multi-view active learning in Table 2 will generate two classifiers
$h_{+}^{s}$ and $h_{-}^{s}$, at least one of which is with error rate no larger than
$R(S^{*})+\epsilon$ with probability at least $1-\delta$. Similarly to the proof of
Theorem~\ref{theorem1+++}, we get Theorem~\ref{theorem3+++} proved.
\end{proof}

As for the situation where $S_{1}^{*}\neq S_{2}^{*}$, similarly to
Theorem~\ref{theorem2++} and Corollary~\ref{corollary1}, we have
Theorem~\ref{theorem4++} and Corollary~\ref{corollary2}.\vspace{-1mm}

\begin{theorem}\label{theorem4++}
For data distribution $\mathcal {D}$ $\alpha$-expanding with respect to hypothesis
class $\mathcal {H}_{1} \times \mathcal {H}_{2}$ according to (\ref{eq6}), if the
multi-view classifiers satisfy $\beta$-condition, by requesting
$\widetilde{O}(\frac{1}{\epsilon})$ labels the multi-view active learning in Table 2
will generate a classifier whose error rate is no larger than $R(S_{1}^{*}\cap
S_{2}^{*})+\epsilon$ with probability at least $1-\delta$.
\end{theorem}

\begin{proof}
Similarly to the proof of Theorem~\ref{theorem2} and Theorem~\ref{theorem3}, we know
that by requesting $\widetilde{O}(\frac{1}{\epsilon})$ labels the multi-view active
learning in Table 2 can get either $R(h_{+}^{s})\leq R(S_{1}^{*}\cap
S_{2}^{*})+\epsilon$ or $R(h_{-}^{s})\leq R(S_{1}^{*}\cap S_{2}^{*})+\epsilon$ with
probability at least $1-\frac{\delta}{2}$. According to Lemma~\ref{theorem1+}, by
requesting $\frac{2\log(\frac{8}{\delta})}{\beta^{2}}$ labels we can decide correctly
whether $R(h_{+}^{s})$ or $R(h_{-}^{s})$ is smaller with probability at least
$1-\frac{\delta}{2}$. Thus, we can get a classifiers whose error rate is no larger than
$R(S_{1}^{*}\cap S_{2}^{*})+\epsilon$ with probability at least $1-\delta$. The total
number of labels to be requested is
$\widetilde{O}(\frac{1}{\epsilon})+\frac{2\log(\frac{8}{\delta})}{\beta^{2}}
=\widetilde{O}(\frac{1}{\epsilon})$.
\end{proof}

\begin{corollary}\label{corollary2}
For data distribution $\mathcal {D}$ $\alpha$-expanding with respect to hypothesis
class $\mathcal {H}_{1} \times \mathcal {H}_{2}$ according to (\ref{eq6}), if the
multi-view classifiers satisfy $\beta$-condition and $Pr(S_{1}^{*}\oplus
S_{2}^{*})\leq\epsilon/2$, by requesting $\widetilde{O}(\frac{1}{\epsilon})$ labels the
multi-view active learning in Table 2 will generate a classifier with error rate no
larger than $R(S_{v}^{*})+\epsilon$ ($v=1,2$) with probability at least $1-\delta$.
\end{corollary}

\begin{proof}
According to Theorem~\ref{theorem4++} we know that by requesting
$\widetilde{O}(\frac{1}{\epsilon})$ labels the multi-view active learning in Table 2
will generate a classifier whose error rate is no larger than $R(S_{1}^{*}\cap
S_{2}^{*})+\frac{\epsilon}{2}$ with probability at least $1-\delta$. With the proof of
Corollary~\ref{corollary1}, we get that $R(S_{1}^{*}\cap S_{2}^{*})+\frac{\epsilon}{2}$
is no larger than $R(S_{v}^{*})+\epsilon$.
\end{proof}

\section{Empirical Verification}
In this section we empirically verify that whether multi-view setting can improve the
sample complexity of active learning in the non-realizable case remarkably.
\begin{figure}[!ht]
\centering
\begin{minipage}[c]{0.4\linewidth}
\centering
\includegraphics[width =1 \linewidth]{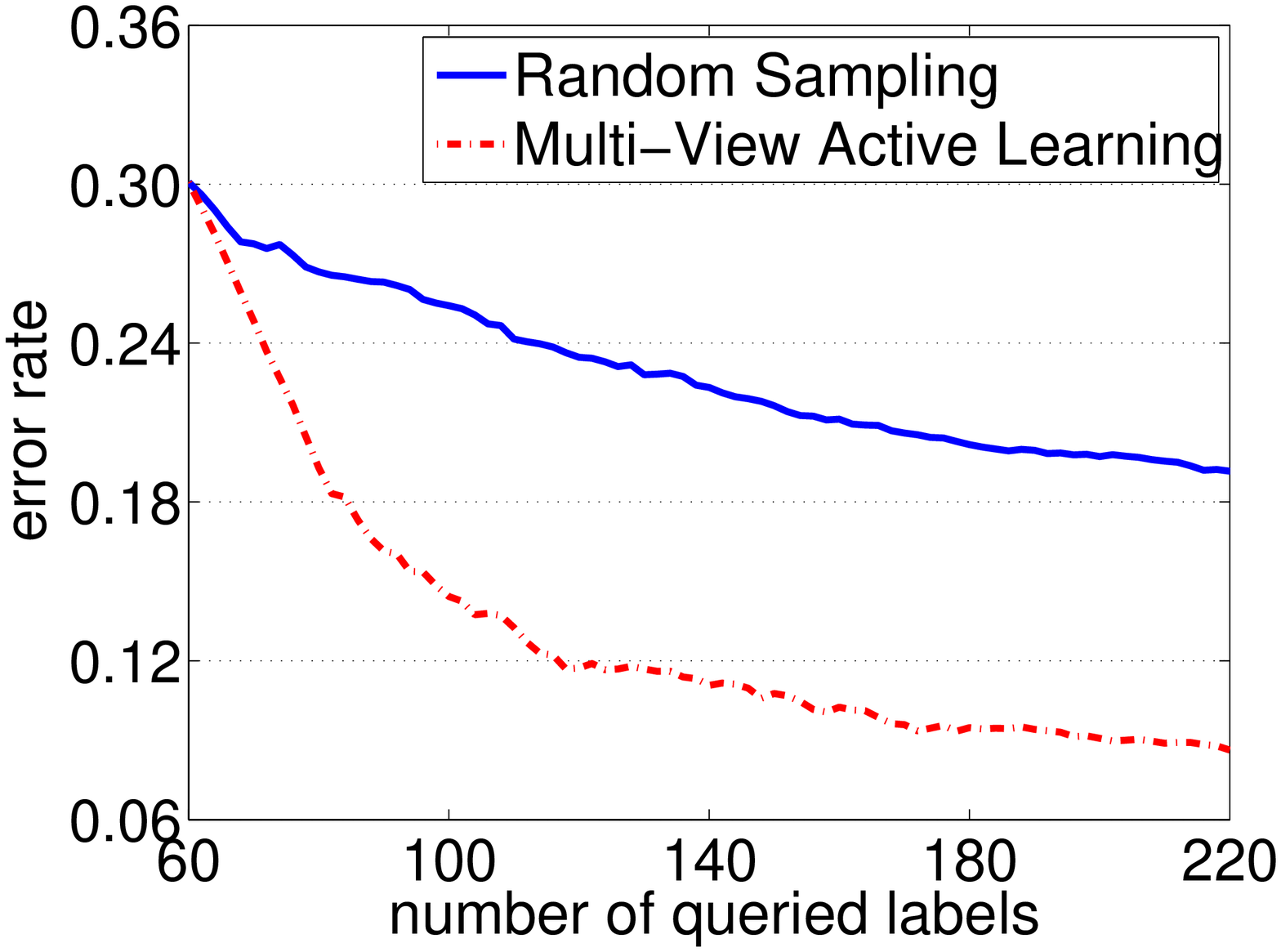}
\centering \mbox{\footnotesize (a) \textit{semi-artificial with 1 cluster}}
\end{minipage}
\centering
\begin{minipage}[c]{0.4\linewidth}
\centering
\includegraphics[width =1 \linewidth]{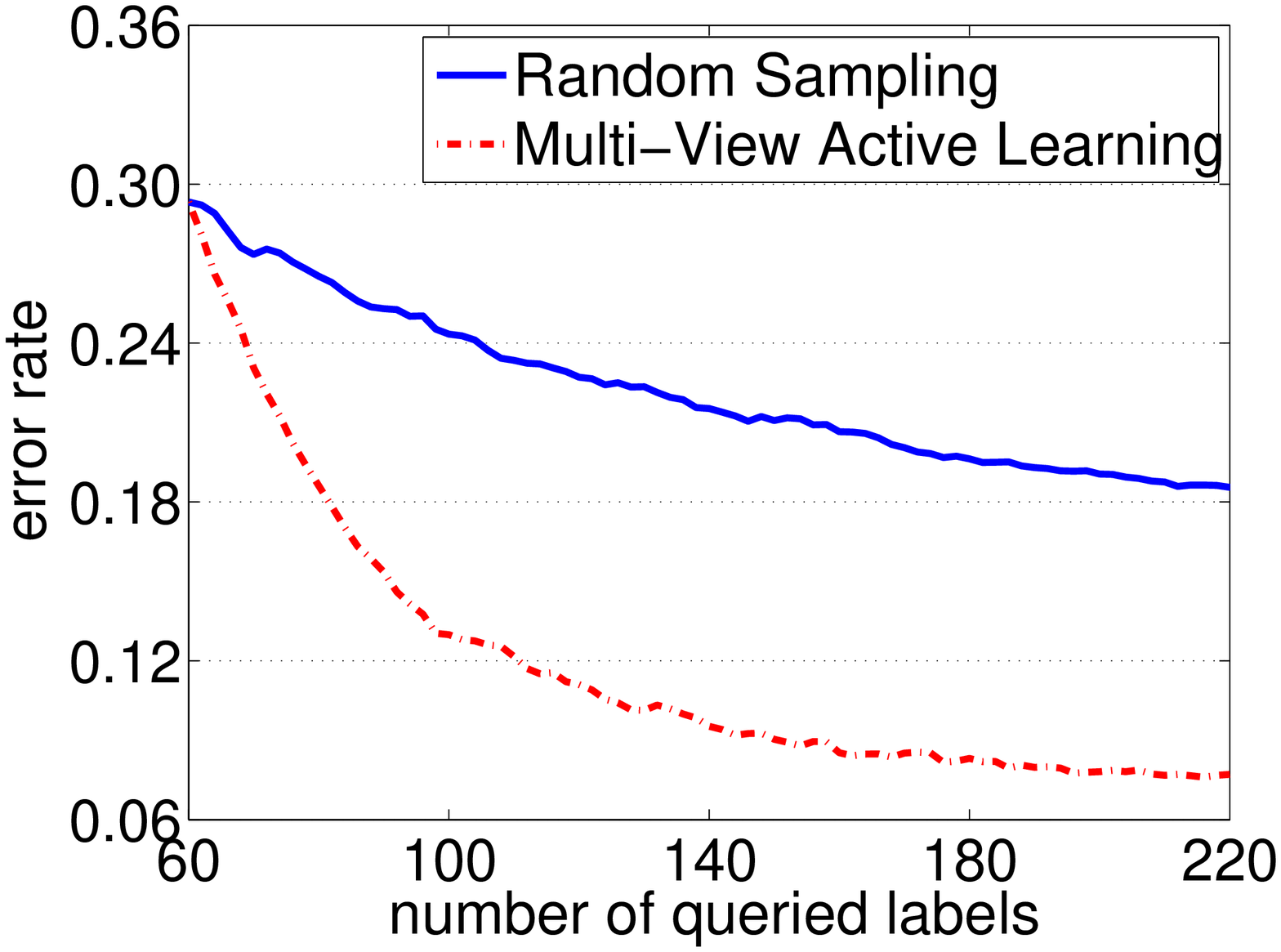}
\centering \mbox{\footnotesize (b) \textit{semi-artificial with 2 clusters}}
\end{minipage}\\
\centering
\begin{minipage}[c]{0.4\linewidth}
\centering
\includegraphics[width =1 \linewidth]{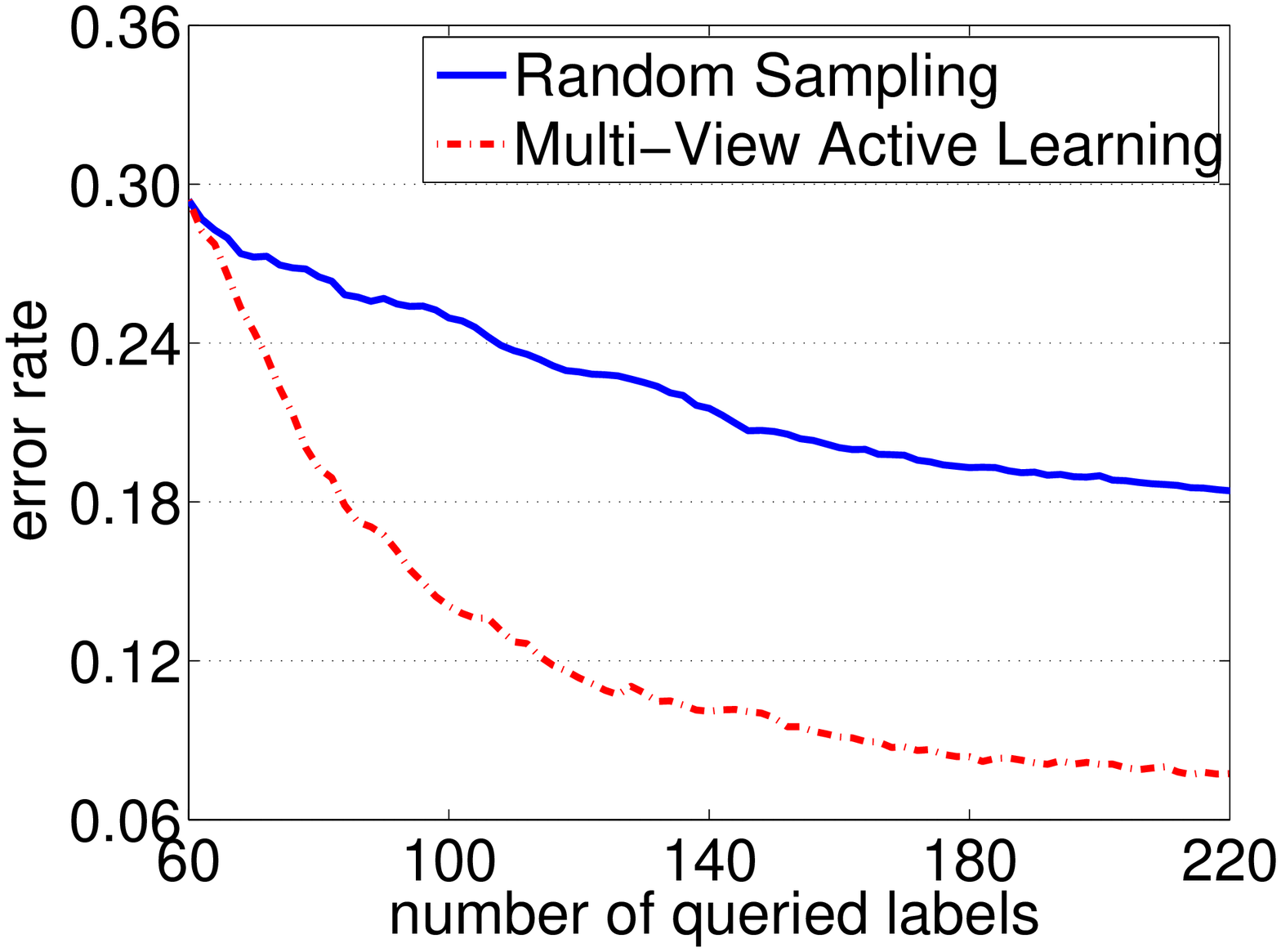}
\centering \mbox{\footnotesize (c) \textit{semi-artificial with 4 clusters}}
\end{minipage}
\centering
\begin{minipage}[c]{0.4\linewidth}
\centering
\includegraphics[width =1 \linewidth]{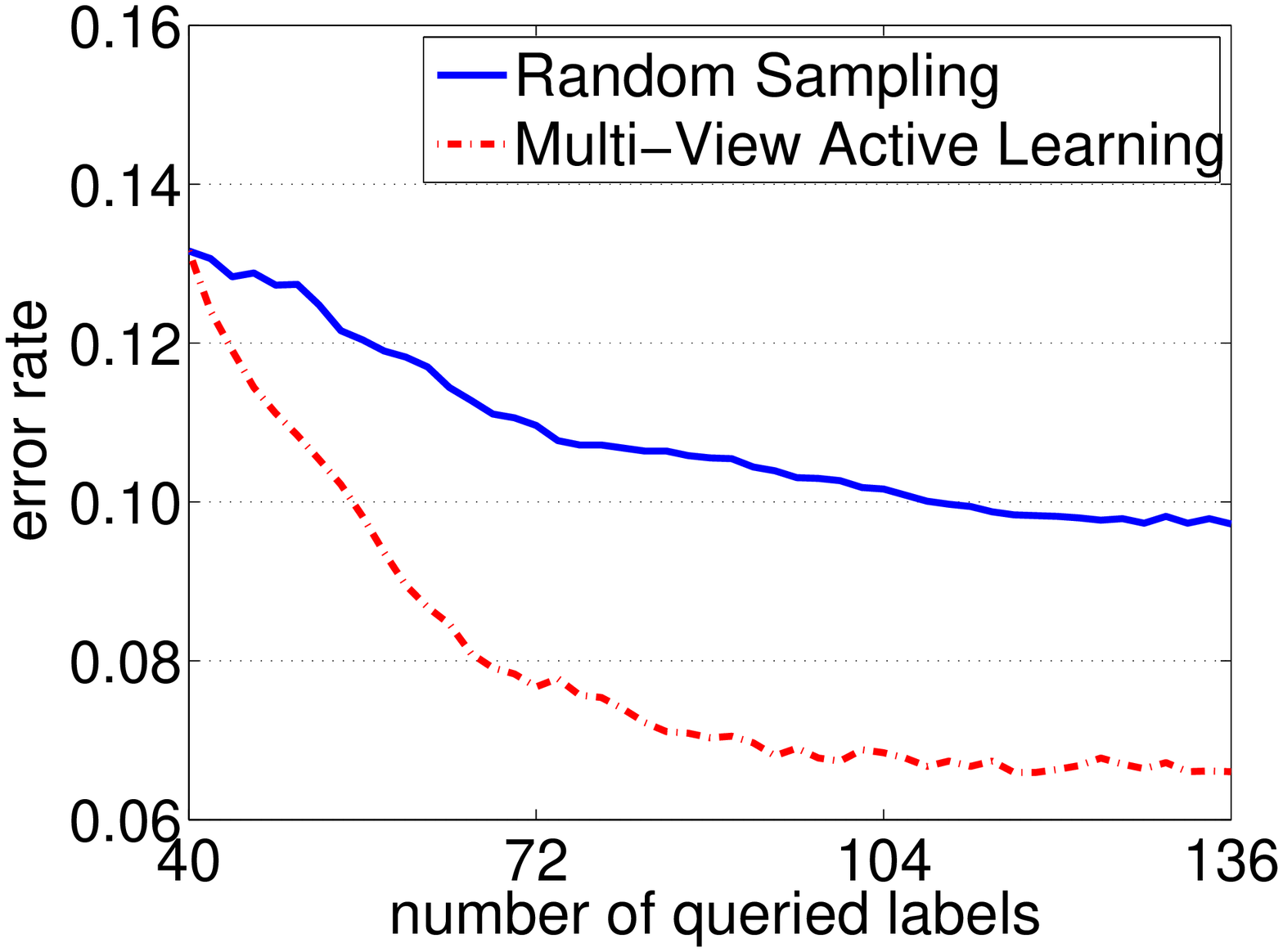}
\centering \mbox{\footnotesize (d) \textit{course}}
\end{minipage}\\
\caption{Multi-view setting improves the sample complexity of active learning in the
non-realizable case remarkably.}\label{Figure1}
\end{figure}

In the experiment we use the \textit{semi-artificial} data set \cite{MusleaMK02} and
the \textit{course} data set \cite{Blum:Mitchell1998}. The \textit{semi-artificial}
data set has two artificial views which are created by randomly pairing two examples
from the same class and contains 800 examples. In order to control the correlation
between the two views, the number of clusters per class can be set as a parameter. We
use 1 cluster, 2 clusters and 4 clusters in the experiments, respectively. The
\textit{course} data set has two natural views: \textit{pages} view (i.e., the text
appearing on the page) and \textit{links} view (i.e., the anchor text attached to
hyper-links pointing to the page) and contains 1,051 examples. We randomly use 25\%
data as the test set and use the remaining 75\% data to generate the unlabeled data set
$\mathcal {U}$. We use Random Sampling as the baseline. In each round, we fix the
number of examples to be queried in Multi-View Active Learning and that in Random
Sampling. Thus, we can study their performances under the same number of queried
examples. In the experiments, we query two examples in each round of the two methods
and implement the classifiers with NaiveBayes in WEKA. The experiments are repeated for
20 runs and Figure \ref{Figure1} plots the average error rates of the two methods
against the number of examples that have been queried. From Figure \ref{Figure1} it can
be found that the performance of Multi-View Active Learning is far better than the
performance of Random Sampling with the same number of queried examples. In other
words, multi-view setting can help improve the sample complexity of active learning in
the non-realizable case remarkably.

\section{Conclusion}
We present the first study on active learning in the non-realizable case under
multi-view setting in this paper. We prove that the sample complexity of multi-view
active learning with unbounded Tsybakov noise can be improved to $\widetilde{O}(\log
\frac{1}{\epsilon})$, contrasting to single-view setting where only \textit{polynomial}
improvement is proved possible with the same noise condition. In general multi-view
setting, we prove that the sample complexity of active learning with unbounded Tsybakov
noise is $\widetilde{O}(\frac{1}{\epsilon})$, where the order of $1/\epsilon$ is
independent of the parameter in Tsybakov noise, contrasting to previous
\textit{polynomial} bounds where the order of $1/\epsilon$ is related to the parameter
in Tsybakov noise. Generally, the non-realizability of learning task can be caused by
many kinds of noise, e.g., misclassification noise and malicious noise. It would be
interesting to extend our work to more general noise model.

\bibliography{NonrealizableMultiview}

\begin{thebibliography}{10}

\bibitem{Anthony1999}
M.~Anthony and P.~L. Bartlett, editors.
\newblock {\em Neural Network Learning: Theoretical Foundations}.
\newblock Cambridge University Press, Cambridge, UK, 1999.

\bibitem{BalcanBL06}
M.-F. Balcan, A.~Beygelzimer, and J.~Langford.
\newblock Agnostic active learning.
\newblock In {\em {ICML}}, pages 65--72, 2006.

\bibitem{Balcan:Blum:Yang2005}
M.-F. Balcan, A.~Blum, and K.~Yang.
\newblock Co-training and expansion: Towards bridging theory and practice.
\newblock In {\em {NIPS} 17}, pages 89--96. 2005.

\bibitem{Balcan2007}
M.-F. Balcan, A.~Z. Broder, and T.~Zhang.
\newblock Margin based active learning.
\newblock In {\em {COLT}}, pages 35--50, 2007.

\bibitem{BalcanCOLT2008}
M.-F. Balcan, S.~Hanneke, and J.~Wortman.
\newblock The true sample complexity of active learning.
\newblock In {\em {COLT}}, pages 45--56, 2008.

\bibitem{Blum:Mitchell1998}
A.~Blum and T.~Mitchell.
\newblock Combining labeled and unlabeled data with co-training.
\newblock In {\em {COLT}}, pages 92--100, 1998.

\bibitem{CastroAllerton06}
R.~M. Castro and R.~D. Nowak.
\newblock Upper and lower error bounds for active learning.
\newblock In {\em {Allerton} {Conference}}, pages 225--234, 2006.

\bibitem{CastroN08}
R.~M. Castro and R.~D. Nowak.
\newblock Minimax bounds for active learning.
\newblock {\em IEEE Transactions on Information Theory}, 54(5):2339--2353,
  2008.

\bibitem{CavallantiCG08}
G.~Cavallanti, N.~Cesa-Bianchi, and C.~Gentile.
\newblock Linear classification and selective sampling under low noise
  conditions.
\newblock In {\em {NIPS} 21}, pages 249--256. 2009.

\bibitem{David94}
D.~A. Cohn, L.~E. Atlas, and R.~E. Ladner.
\newblock Improving generalization with active learning.
\newblock {\em Machine Learning}, 15(2):201--221, 1994.

\bibitem{Dasgupta05}
S.~Dasgupta.
\newblock Analysis of a greedy active learning strategy.
\newblock In {\em {NIPS} 17}, pages 337--344. 2005.

\bibitem{DasguptaNIPS05}
S.~Dasgupta.
\newblock Coarse sample complexity bounds for active learning.
\newblock In {\em {NIPS} 18}, pages 235--242. 2006.

\bibitem{DasguptaNIPS07}
S.~Dasgupta, D.~Hsu, and C.~Monteleoni.
\newblock A general agnostic active learning algorithm.
\newblock In {\em {NIPS} 20}, pages 353--360. 2008.

\bibitem{Dasgupta2005}
S.~Dasgupta, A.~T. Kalai, and C.~Monteleoni.
\newblock Analysis of perceptron-based active learning.
\newblock In {\em {COLT}}, pages 249--263, 2005.

\bibitem{DEVROYE1996}
L.~Devroye, L.~Gy{\"o}rfi, and G.~Lugosi, editors.
\newblock {\em A Probabilistic Theory of Pattern Recognition}.
\newblock Springer, New York, 1996.

\bibitem{Freund97}
Y.~Freund, H.~S. Seung, E.~Shamir, and N.~Tishby.
\newblock Selective sampling using the query by committee algorithm.
\newblock {\em Machine Learning}, 28(2-3):133--168, 1997.

\bibitem{Hanneke07}
S.~Hanneke.
\newblock A bound on the label complexity of agnostic active learning.
\newblock In {\em {ICML}}, pages 353--360, 2007.

\bibitem{Hanneke2009}
S.~Hanneke.
\newblock Adaptive rates of convergence in active learning.
\newblock In {\em {COLT}}, 2009.

\bibitem{Kaariainen06}
M.~K{\"a}{\"a}ri{\"a}inen.
\newblock Active learning in the non-realizable case.
\newblock In {\em {ACL}}, pages 63--77, 2006.

\bibitem{MusleaMK02}
I.~Muslea, S.~Minton, and C.~A. Knoblock.
\newblock Active + semi-supervised learning = robust multi-view learning.
\newblock In {\em {ICML}}, pages 435--442, 2002.

\bibitem{Tsybakov04}
A.~Tsybakov.
\newblock Optimal aggregation of classifiers in statistical learning.
\newblock {\em The Annals of Statistics}, 32(1):135--166, 2004.

\bibitem{wangNIPS2009}
L.~Wang.
\newblock Sufficient conditions for agnostic active learnable.
\newblock In {\em {NIPS} 22}, pages 1999--2007. 2009.

\bibitem{WangZ08}
W.~Wang and Z.-H. Zhou.
\newblock On multi-view active learning and the combination with
  semi-supervised learning.
\newblock In {\em {ICML}}, pages 1152--1159, 2008.

\end{thebibliography}
\bibliographystyle{plain}

\end{document}